\documentclass[final, 12pt]{colt2018}

\usepackage{times}
\usepackage{booktabs}       
\usepackage{amsfonts}       
\usepackage{nicefrac}       
\usepackage{enumitem}

\usepackage{amsmath,amssymb,mathrsfs}
\PassOptionsToPackage{ruled, vlined, linesnumbered}{algorithm2e}
\usepackage{algorithm}
\usepackage[noend]{algpseudocode}

\usepackage{graphicx}
\usepackage{mathtools}
\usepackage{footnote}
\usepackage{float}
\usepackage{xspace}
\usepackage{multirow}
\usepackage{wrapfig}
\usepackage{framed}

\usepackage{footnote}
\makesavenoteenv{tabular}
\makesavenoteenv{table}

\newcommand{\TVD}[1]{\norm{#1}_\text{TV}}

\newcommand{\EG}{\textsc{Epoch-Greedy}\xspace}
\newcommand{\EPG}{\textsc{$\epsilon$-Greedy}\xspace}
\newcommand{\minimonster}{\textsc{ILOVETOCONBANDITS}\xspace}

\newcommand{\AdaEG}{\textsc{Ada-Greedy}\xspace}
\newcommand{\AdaILTCB}{\textsc{Ada-ILTCB}\xspace}

\newcommand{\AdaBIN}{\textsc{Ada-BinGreedy}\xspace}
\newcommand{\corral}{\textsc{Corral}\xspace}
\newcommand{\bistro}{\textsc{BISTRO+}\xspace}
\newcommand{\base}[1]{{{\cal{B}}_{#1}}}
\newcommand{\scale}{\rho}
\newcommand{\test}{\textsc{NonstatTest}\xspace}
\newcommand{\true}{\textit{True}\xspace}
\newcommand{\false}{\textit{False}\xspace}
\newcommand{\flag}{\textsc{flag}\xspace}

\newcommand{\calX}{{\mathcal{X}}}
\newcommand{\calS}{{\mathcal{S}}}
\newcommand{\calI}{{\mathcal{I}}}
\newcommand{\calJ}{{\mathcal{J}}}

\newcommand{\calD}{{\mathcal{D}}}
\newcommand{\calE}{{\mathcal{E}}}
\newcommand{\calR}{{\mathcal{R}}}

\newcommand{\avgR}{\wh{\cal{R}}}
\newcommand{\ips}{\wh{r}}
\newcommand{\whpi}{\wh{\pi}}

\newcommand{\whV}{\wh{V}}
\newcommand{\Reg}{\text{\rm Reg}}
\newcommand{\whReg}{\wh{\text{\rm Reg}}}

\newcommand{\one}{\boldsymbol{1}}
\newcommand{\var}{\Delta}
\newcommand{\bvar}{\bar{\Delta}}
\newcommand{\p}{\prime}
\newcommand{\evt}{\textsc{Event}}

\DeclareMathOperator*{\argmax}{argmax}

\DeclarePairedDelimiter\abs{\lvert}{\rvert}
\DeclarePairedDelimiter\bigabs{\big\lvert}{\big\rvert}
\DeclarePairedDelimiter\ceil{\lceil}{\rceil}
\DeclarePairedDelimiter\floor{\lfloor}{\rfloor}

\newcommand{\field}[1]{\mathbb{#1}}

\newcommand{\fR}{\field{R}}

\newcommand{\E}{\field{E}}

\newcommand{\norm}[1]{\left\|{#1}\right\|}

\newcommand{\scO}{\mathcal{O}}

\newcommand{\wh}{\widehat}

\newtheorem{cor}{Corollary}

\newcommand{\order}{\ensuremath{\mathcal{O}}}
\newcommand{\otil}{\ensuremath{\widetilde{\mathcal{O}}}}

\title{Efficient Contextual Bandits in Non-stationary Worlds}

 \coltauthor{\Name{Haipeng Luo} \Email{haipengl@usc.edu}\\
 \addr University of Southern California
 \AND
 \Name{Chen-Yu Wei} \Email{chenyu.wei@usc.edu}\\
 \addr University of Southern California
 \AND
 \Name{Alekh Agarwal} \Email{alekha@microsoft.com} \\
 \addr Microsoft Research, NYC
 \AND 
\Name{John Langford}\Email{jcl@microsoft.com} \\
 \addr Microsoft Research, NYC
 }

\begin{document}
\maketitle

 \begin{abstract}
  Most contextual bandit algorithms minimize regret against the best fixed
  policy, a questionable benchmark for non-stationary environments that are
  ubiquitous in applications.  In this work, we develop several efficient
  contextual bandit algorithms for non-stationary environments by equipping
  existing methods for i.i.d. problems with sophisticated statistical tests so as to dynamically    
  adapt to a change in distribution.  
  
  We analyze various standard notions of regret suited to non-stationary environments for these  
  algorithms, including interval regret, switching regret, and dynamic regret.
  When competing with the best policy at each time, one of our algorithms achieves regret $\order(\sqrt{ST})$
  if there are $T$ rounds with $S$ stationary periods, 
  or more generally $\order(\var^{1/3}T^{2/3})$ where $\var$ is some non-stationarity measure.
  These results almost match the optimal guarantees achieved by an inefficient baseline that is a variant of the classic Exp4 algorithm.
  The dynamic regret result is also the first one for efficient and fully adversarial contextual bandit.
  
  Furthermore, while the results above require tuning a parameter based on the unknown quantity $S$ or $\var$,
  we also develop a parameter free algorithm achieving regret $\min\{S^{1/4}T^{3/4}, \var^{1/5}T^{4/5}\}$.
  This improves and generalizes the best existing result $\var^{0.18}T^{0.82}$ 
  by~\citet{KarninAn16} which only holds for the two-armed bandit problem.
\end{abstract}
\section{Introduction}

Algorithms for the contextual bandit problem have been developed for
adversarial~\citep{AuerCeFrSc02},
stochastic~\citep{AgarwalHsKaLaLiSc14,LangfordZh08} and
hybrid~\citep{RakhlinSr16,SyrgkanisLuKrSc16} environments. Despite the
specific setting, however, almost all these works minimize the
classical notion of regret that compares the reward of the algorithm
to the \emph{best fixed policy in hindsight}. This is a natural
benchmark when the data generating mechanism is essentially
stationary, so that a fixed policy can attain a large reward. However,
in many applications of contextual bandits, we are faced with an
extremely non-stationary world. For instance, the pool of available
news stories or blog articles rapidly evolves in content
personalization domains, and people's preferences typically exhibit
trends on daily, weekly and seasonal scales. In such cases, one wants
to compete with an appropriately adaptive sequence of benchmark
policies, for the baseline to be meaningful.

Prior works in a context-free setting (that is, the multi-armed bandit
problem) have studied regret to a sequence of actions, whenever that
sequence is \emph{slowly changing} under some appropriate measure (see
e.g.~\citep{AuerCeFrSc02, BesbesGuZe14, BesbesGuZe15, KarninAn16,
  WeiHoLu16}). A natural generalization to the contextual setting
would be to compete with a sequence of policies, all chosen from some
policy class. 
Extension of the prior context-free works to the contextual
setting indeed yields algorithms with such guarantees, as we show with
a baseline example (Exp4.S). However, the computation and storage of
the resulting algorithms are both linear in the cardinality of the
policy class, making tractable implementation impossible except for
very small policy classes.

To overcome the computational obstacle, all previous works on
efficient contextual bandits assume access to an optimization oracle
which can find the policy with the largest reward on any dataset
containing context-reward
pairs~\citep{LangfordZh08,AgarwalHsKaLaLiSc14,RakhlinSr16,SyrgkanisLuKrSc16}.
Given such an oracle, however, it is known that no efficient
low-regret algorithms exist in the fully adversarial
setting~\citep[Theorem 25]{Hazan2016}, even without any challenges of
non-stationarity. Consequently all previous works explicitly rely on
assumptions such as i.i.d. contexts, or even i.i.d. context-reward
pairs. 

As a warm-up and also an example to show the difficulty of the problem, 
we first consider a general approach to convert an algorithm for the stationary setting
to an algorithm that can deal with non-stationary data.
The idea is to combine different copies of the base algorithm,
each of which starts at a different time to learn over different data segments.
This can be seen as a natural generalization of the approach of~\citet{HazanSe07} for the full information
setting. 
We build on a recent result of~\citet{AgarwalLuNeSc17} to deal with the additional challenges due to partial
feedback and use \bistro~\citep{SyrgkanisLuKrSc16} as the base algorithm
since it is efficient and requires no statistical assumption on the rewards. 
However, unlike the full information setting, the regret rates degrade after this conversion as we show,
making this general approach unsatisfying.


We next consider a more specific approach by equipping existing algorithms for the i.i.d. setting,
such as \EG~\citep{LangfordZh08} and the statistically more efficient approach of~\cite{AgarwalHsKaLaLiSc14},
with some sophisticated statistical tests to detect non-stationarity
(the resulting algorithms are called \AdaEG and \AdaILTCB respectively).
Once such non-stationarity is detected, the algorithms restart from scratch.
The exact tests are algorithm-specific and based
on verifying certain concentration inequalities which the algorithm
relies upon, but the general idea might be applicable to extending
other contextual bandit algorithms as well.


We present strong theoretical guarantees for our algorithms, 
in terms of interval regret, switching regret and dynamic
regret (defined in Section~\ref{sec:setup}). 
A high-level outcome of our analysis is that the algorithms enjoy a regret bound on
any time interval that is sufficiently stationary (called interval
regret), compared with the best fixed policy for that interval. 
This general result has important corollaries, discussed in
Section~\ref{sec:implications}. For example, if the data-generating
process is typically i.i.d., except there are \emph{hard switches} in
the data distribution every so often, then our algorithms perform as
if they knew the change points in advance, up to a small penalty in
regret (called switching regret). More generally, if the data
distribution is \emph{slowly drifting}, we can still provide
meaningful regret bounds (called dynamic regret) when competing to the
best policy at each time (instead of a fixed policy across all rounds).

These results are summarized in Table~\ref{tab:results}.
The highlight is that our computationally efficient algorithm \AdaILTCB enjoys almost the same guarantee
as the inefficient baseline Exp4.S for all three regret measures,
which is optimal in light of the existing results for the special case of multi-armed bandit.
Importantly, the dynamic regret bounds for our algorithms hold under
a {\it fully adversarial} setting.\footnote{Note that this does not contradict with the hardness results in~\citep{Hazan2016}
since the bound is data-dependent and could be linear in $T$ in the worst case.}
As far as we know, this is the first result on adversarial and efficient contextual bandits.

All the results above, including those for Exp4.S, require tuning a parameter
in terms of some unknown quantity. Otherwise the results degrade as shown in Table~\ref{tab:results}
and become vacuous when the non-stationarity measure (the number of stationary periods $S$ or the reward variation $\var$) is large.
Our final contribution is a parameter-free variant of \AdaEG, called \AdaBIN,
which achieves better regret (even compared to Exp4.S) in the regime when $S$ or $\var$ is large and unknown.
Importantly, this result even improves upon the best existing result by~\citet{KarninAn16} for the context-free setting,
where a regret bound of order $\var^{0.18}T^{0.82}$ is shown for the two-armed bandit problem.
We improve the bound to $\min\{S^{1/4}T^{3/4}, \var^{1/5}T^{4/5}\}$ and also significantly generalize it to 
the multi-armed and contextual setting.


\setcounter{footnote}{0}
\renewcommand{\arraystretch}{1.5}
\begin{table}[t]
\centering
\caption[Comparisons]{Comparisons of different results presented in this
  work. ``OE?'' indicates whether the algorithm is Oracle-Efficient or not. 
  $T$ is the total number of rounds, $\calI$ is the interval on which interval regret is measured,
  $S$ is the number of i.i.d. periods, $\var$ is the reward variation,
  and $\bvar \geq \var$ is the total variation, all defined in
  Section~\ref{sec:setup}. These parameters are assumed to be known
  for the column ``tuned'' but unknown for the column ``param-free''.\footnotemark  
  \;Dependence on other parameters are omitted. Results for \bistro assumes a transductive setting,
  and interval regret for the last three algorithms assumes (approximately) i.i.d. data on $\calI$.}
\label{tab:results}
\begin{tabular}{|c|c|c|c|c|c|c|c|}
\hline
\multirow{ 2}{*}{\small Algorithm} & \multirow{ 2}{*}{\small OE?} & \multicolumn{2}{c|}{\small Interval Regret} & 
\multicolumn{2}{c|}{\small Switching Regret} & \multicolumn{2}{c|}{\small Dynamic Regret} \\
\cline{3-8}
& & param-free & tuned & param-free & tuned & param-free & tuned \\
\hline
{\small Exp4.S (baseline)} & {\small N} & $\sqrt{T}$ & $\sqrt{|\calI|}$  & 
$S\sqrt{T}$ & $\sqrt{ST} $ & 
$\sqrt{\var}T^\frac{2}{3}$ & $\var^{\frac{1}{3}}T^{\frac{2}{3}} $ \\
\hline
{\small Corral \bistro} & {\small Y} & $T^{\frac{3}{4}}$ & $T^{\frac{1}{4}}\sqrt{|\calI|}$ & 
$ST^{\frac{3}{4}}$ & $\sqrt{S}T^{\frac{3}{4}}$ & $\sqrt{\var} T^{\frac{5}{6}}$ & $\var^\frac{1}{3}T^{\frac{5}{6}}$ \\
\hline
{\small \AdaEG} & {\small  Y} &  $T^{\frac{1}{6}}\sqrt{|\calI|}$ & $|\calI|^\frac{2}{3} $ &
$ \sqrt{S}T^\frac{2}{3}$ & $S^\frac{1}{3}T^\frac{2}{3} $ &
$\sqrt{\var}T^\frac{3}{4}$ &  $\var^\frac{1}{4}T^\frac{3}{4} $ \\
\hline
{\small \AdaILTCB} & {\small Y} & $\sqrt{T}$ & $\sqrt{|\calI|}$  & 
$S\sqrt{T}$ & $\sqrt{ST} $ & 
$ \bvar T^\frac{2}{3}$ & $\bvar^{\frac{1}{3}}T^{\frac{2}{3}} $ \\
\hline
{\small \AdaBIN} & {\small Y} & \multicolumn{2}{c|}{$T^{\frac{3}{4}}$}  & 
\multicolumn{2}{c|}{$S^{\frac{1}{4}}T^{\frac{3}{4}}$} & \multicolumn{2}{c|}{$\var^{\frac{1}{5}}T^{\frac{4}{5}}$} \\
\hline
\end{tabular}

\footnotetext{Other (incomparable) bounds for the ``param-free''  column are also possible. See discussions in respective sections.}

\end{table}

\paragraph{Related work.}
The idea of testing for non-stationarity in bandits was studied
in~\citep{BubeckSl12} and~\citep{AuerCh16} for a very different
purpose. 
The closest bounds to those in Table~\ref{tab:results} are in
the non-contextual setting~\citep{AuerCeFrSc02, BesbesGuZe14,
  BesbesGuZe15, WeiHoLu16} as mentioned earlier. \citet{chakrabarti2009mortal} study a context-free setup where
the action set changes. To the best our knowledge, oracle-efficient
contextual bandit algorithms for non-stationary environments were only
studied before in~\citep{SyrgkanisKrSc16}, where a reduction from
competing with a switching policy sequence to competing with a fixed
policy was proposed.  However, the reduction cannot be applied to the
i.i.d methods~\citep{LangfordZh08, AgarwalHsKaLaLiSc14}, and it heavily
relies on knowing the number of switches and the transductive setting. Additionally,
this approach gives no guarantees on interval regret or dynamic
regret, unlike our results.

 \section{Preliminaries}\label{sec:setup}

The contextual bandits problem is defined as follows.  Let $\calX$ be
an arbitrary context space and $K$ be the number of actions.  Let
$[n]$ denote the set $\{1, \ldots, n\}$ for any integer $n$.  A
mapping $\pi: \calX \rightarrow [K]$ is called a policy and the
learner is given a fixed set of policies $\Pi$.  For simplicity, we
assume $\Pi$ is a finite set but with a large cardinality $N = |\Pi|$.
Ahead of time, the environment decides $T$ distributions $\calD_1,
\ldots, \calD_T$ on $\calX \times [0,1]^K$, and draws $T$
context-reward pairs $(x_t, r_t) \sim \calD_t$ for $t=1,\ldots, T$
independently.\footnote{That is, the data generating process is
  oblivious to the algorithm.}  Then at each round $t = 1, \ldots, T$,
the environment reveals $x_t$ to the learner, the learner picks an
action $a_t \in [K]$ and observes its reward $r_t(a_t)$.  The regret
of the learner with respect to a policy $\pi$ at round $t$ is
$r_t(\pi(x_t)) - r_t(a_t)$.  Most existing results on contextual
bandits focus on minimizing cumulative regret against any fixed policy
$\pi \in \Pi$: $\sum_{t=1}^T r_t(\pi(x_t)) - r_t(a_t)$.

To better deal with non-stationary environments, we consider several
related notions of regret.  The first one is cumulative regret with
respect to a fixed policy on a time interval $\calI$, which we call
{\it interval regret} on $\calI$. Specifically, we use the notation
$\calI = [s, s']$ for $s \leq s'$ and $s, s' \in [T]$ to denote the
set $\{s, s+1, \ldots, s'\}$ and call it a time interval (starting
from round $s$ to round $s'$).  The regret with respect to a fixed
$\pi \in \Pi$ on a time interval $\calI$ is then defined as $\sum_{t
  \in \calI} r_t(\pi(x_t)) - r_t(a_t)$. 
This is similar to the notion of adaptive and strongly
adaptive regret~\citep{HazanSe07, DanielyGoSh15}. We use the term
interval regret without any specific interval when the choice is
clear from context.

Interval regret is useful in studying more general regret measures for non-stationary environments.
Specifically, we aim at the most challenging benchmark,
that is, the cumulative rewards achieved by using the best policy {\it at each time}.
Formally, let $\calR_t(\pi) \coloneqq \E_{(x,r)\sim\calD_t} r(\pi(x))$ be the expected reward of policy
$\pi$ under $\calD_t$ and $\pi^\star_t \coloneqq \argmax_{\pi \in \Pi} \calR_t(\pi)$ be the optimal policy at round $t$.
Then the aforementioned general regret is defined as $\sum_{t=1}^T r_t(\pi^\star_t(x_t)) - r_t(a_t)$.
It is well-known that in general no sub-linear regret is achievable with this definition. 

However, one can bound such regret in terms of some quantity that measures the non-stationarity of the environment
and achieve meaningful results whenever such quantity is not too large.
One example is to count the number of switches in the distribution sequence,
that is, $\sum_{t=2}^T \one\{\calD_t \neq \calD_{t-1}\}$.
We denote this by $S - 1$ (so that $S$ is the number of i.i.d. segments)
and call a regret bound in terms of $S$ {\it switching regret}.

Switching regret might be meaningless if the distribution is slowly drifting,
leading to a large number of switches but overall a small amount of variation in the distribution.
To capture this situation, we also consider another type of non-stationarity measure,
generalizing a similar notion from the multi-armed bandit literature~\citep{BesbesGuZe14}.
Specifically, define \sloppy$\var = \sum_{t=2}^T \max_{\pi \in \Pi} |\calR_t(\pi) - \calR_{t-1}(\pi)|$
to be the variation of reward distributions.
Note that this is a lower bound on the sum of total variation between consecutive distributions
$\bvar = \sum_{t=2}^T \TVD{\calD_t-\calD_{t-1}} = \sum_{t=2}^T\int_{[0,1]^K}\int_{\calX} \bigabs{\calD_{t}(x,r)-\calD_{t-1}(x,r)} dxdr$
(see Lemma~\ref{lemma:variation_relation} for a proof).
We call regret bounds in terms of $\var$ or $\bvar$ {\it dynamic regret}.

All algorithms we consider construct a distribution $p_t$ over actions
at round $t$ and then sample $a_t \sim p_t$.  The importance weighted
reward estimator is defined as $\ips_t(a) = \frac{r_t(a)}{p_t(a)}
\one\{a = a_t\},\; \forall a \in [K]$.  For an interval $\calI$, we
use $\calR_\calI(\pi)$ and $\avgR_\calI(\pi)$ to denote the average
expected and empirical rewards of $\pi$ over $\calI$ respectively,
that is, $\calR_\calI(\pi) =
\frac{1}{|\calI|}\sum_{t\in\calI}\calR_t(\pi)$ and $\avgR_\calI(\pi) =
\frac{1}{|\calI|}\sum_{t\in\calI}\ips_t(\pi(x_t))$.  
The empirically best policy on interval $\calI$ is defined as $\hat{\pi}_\calI = \argmax_{\pi\in\Pi}\avgR_\calI(\pi)$.
The number of i.i.d. periods, the reward variation, and the total variation on an interval
$\calI = [s,s']$ are respectively defined as 
$S_\calI = 1+\sum_{\tau=s+1}^{s'} \one\{\calD_\tau \neq \calD_{\tau-1}\}$,
$\var_\calI \coloneqq \sum_{\tau = s+1}^{s'} \max_{\pi \in \Pi} |\calR_\tau(\pi) - \calR_{\tau - 1}(\pi)|$,
and $\bvar_\calI \coloneqq \sum_{\tau = s+1}^{s'}\TVD{\calD_{\tau}-\calD_{\tau-1}}$.

We use $\calD_t^\calX$ to denote
the marginal distribution of $\calD_t$ over $\calX$, and $\E_t$ to
denote the conditional expectation given everything before round $t$.
Finally, we are interested in efficient algorithms assuming access to
an optimization oracle~\citep{AgarwalHsKaLaLiSc14}:

\begin{definition}
The argmax oracle (AMO) is an algorithm which takes any set $\calS$ of
context-reward pairs $(x, r) \in \calX \times \fR^K$ as
inputs and outputs any policy in $\argmax_{\pi \in \Pi} \sum_{(x,r)\in
  \calS} r(\pi(x))$.
\end{definition}

An algorithm is oracle-efficient if its total running time and the
number of oracle calls are both polynomial in $T, K$ and $\ln N$,
excluding the running time of the oracle itself. 

In the rest of the paper, we start with discussing interval regret in Section~\ref{sec:interval_regret},
followed by the implications for switching/dynamic regret in Section~\ref{sec:implications}.
The parameter-free algorithm \AdaBIN is then discussed in Section~\ref{sec:bin}.

\section{Interval Regret}\label{sec:interval_regret}

In this section we present several algorithms with interval regret
guarantees.  As a starter and a baseline, we first point out that a
generalization of the Exp3.S algorithm~\citep{AuerCeFrSc02} and
Fixed-Share~\citep{HerbsterWa98} to the contextual bandit setting,
which we call Exp4.S, already provides a strong interval regret
guarantee as shown by the following theorem. We include the algorithm
and the proof in Appendix~\ref{app:Exp4.S}.  Crucially, Exp4.S
requires maintaining weights for each policy and is thus not oralce-efficient.

\begin{theorem}\label{thm:Exp4.S}
Exp4.S with parameter $L$ ensures that for any time interval $\calI$
such that $|\calI| \leq L$, we have $\E\left[
  \sum_{t\in\calI}r_t(\pi(x_t)) - r_t(a_t) \right] \leq
\order(\sqrt{LK\ln(NL)})$ for any $\pi \in \Pi$, where the expectation
is with respect to the randomness of both the algorithm and the
environment.
\end{theorem}

Note that in bandit settings, it is impossible to achieve regret
$\order(\sqrt{|\calI|})$ for all interval $\calI$ simultaneously~\citep{DanielyGoSh15}.
When $|\calI|$ is unknown, a safe choice is to pick $L = T$
(this is how we obtain the results in the ``param-free'' column of Table~\ref{tab:results} for interval regret).
Next we prove statements similar to Theorem~\ref{thm:Exp4.S} but with
oracle-efficient algorithms. 

\paragraph{A general approach.}
In the full information setting, a general approach to convert an algorithm with classic regret guarantee
to another with interval regret is to combine different copies of the algorithm with an expert algorithm,
each of which starts at a different time step to learn over different time intervals.
This works well in the full information setting where one has correct feedback to update all the base algorithms,
but becomes challenging in the bandit setting.
We show in Appendix~\ref{app:BISTRO+} how to leverage recent results by~\citet{AgarwalLuNeSc17}
and~\citet{wei2018more} to deal with such challenges.
As an example we use the \bistro algorithm~\citep{SyrgkanisLuKrSc16, rakhlin2016bistro} as the base algorithm
since it is oracle-efficient and allows adversarial rewards.

\begin{theorem}\label{thm:corralling_BISTRO+}
In the transductive setting, Algorithm~\ref{alg:corralling_BISTRO+} in Appendix~\ref{app:BISTRO+} guarantees
that for any time interval $\calI$ such that $|\calI| \leq L$ and any policy $\pi \in \Pi$, we have $\E\left[
  \sum_{t\in\calI}r_t(\pi(x_t)) - r_t(a_t) \right] \leq
\otil(T^{\frac{1}{4}}(LK)^{\frac{1}{2}}(\ln N)^{\frac{1}{4}})$.
\end{theorem}

Unlike the full information setting, this general approach results in worse regret rates and is unsatisfying
(\bistro achieves $\otil(T^{2/3})$ for the classic regret and here we only obtain $\otil(T^{3/4})$).
In the following subsections, we turn to different approaches. 

\subsection{\AdaEG}

The simplest oracle-efficient contextual bandit algorithm is the \EG
method~\citep{LangfordZh08} which assumes i.i.d. data. 
In this section, we extend the related \EPG
algorithm to enjoy a small interval regret on any interval with a
small variation.

\EPG plays uniformly at random with a small probability and
otherwise follows the empirically best policy $\pi_t=\argmax_{\pi\in\Pi}
\avgR_{[1,t-1]}(\pi)$. The number of oracle calls can be greatly reduced if the learner updates the best policy only at $t=1, 2, 4, 8, \ldots$ (that is, $\pi_t=\argmax_{\pi\in\Pi}
\avgR_{[1,2^{\floor{\log_2{t}}}-1]}(\pi)$). 
\AdaEG, described in Algorithm~\ref{alg:AdaGreedy2}, behaves similarly to this version of \EPG. 
The difference is that at each round, an additional \textit{non-stationarity test} is executed. 
The test monitors whether there is a policy performing significantly better on recent samples (collected in a doubling manner), 
compared to the policy that the algorithm is using. 
Intuitively, if such a policy exists, there should have been a significant shift in the distribution.
In this case, the algorithm restarts from scratch. 

In addition, the algorithm also resets every $L$ rounds for some parameter $L$ (Line~\ref*{line:trigger_rerun}). 
This prevents the risk of slow detection of a distribution change, but at the same time also causes some extra penalty when the environment is stationary.
The parameter $L$ trades these two kinds of costs, and can be selected based on prior knowledge about the environment. 

We call the rounds between resets an {\it epoch} (so epoch $i$ is the interval $[T_i+1, T_{i+1}]$), 
and the rounds between updates of the empirically best policy a {\it block} (so block $j$ of epoch $i$ is the interval $[T_i+2^{j-1}, T_i+2^{j}-1]$).

Note that there are only two places where we need to invoke the oracle: computing $\hat{\pi}_{(i,j)}$ and $\hat{\pi}_{A}$ 
($\hat{\pi}_{B}$ is simply equal to $\hat{\pi}_{(i,j)}$),
and it is thus clear that at most $\order(\ln L) = \order(\ln T)$ oracle calls are used per round.


We prove the following result for \AdaEG, stating a regret bound for all intervals with length smaller than $L$
and variation smaller than another parameter $v$ of the algorithm.
\begin{theorem}\label{thm:AdaEG2}
With probability at least $1 - \delta$, for all time intervals $\calI$
such that $|\calI| \leq L$ and $\var_\calI \leq v$, \AdaEG with
parameters $L$, $v$ and $\delta$ guarantees for any $\pi \in
\Pi$,%
\footnote{We use notation $\otil$ to suppress dependence on
  logarithmic factors in $L, T, K$ and $\ln(N/\delta)$. } 
\[
\sum_{t \in \calI} r_t(\pi(x_t)) - r_t(a_t) \leq \otil\left(\abs{\calI}v +
L^\frac{1}{6}\sqrt{K\abs{\calI}\ln(N/\delta)} + K\ln(N/\delta)\right).
\]
\end{theorem}

\begin{algorithm}[t]
\SetAlgoLined
\setcounter{AlgoLine}{0}
\SetAlgoVlined
\DontPrintSemicolon
\caption{\AdaEG}\label{alg:AdaGreedy2}
\nl {\bf Input}: largest allowed interval length $L$ and variation $v$, allowed failure probability $\delta$  \\ 
\nl {\bf Define}:
$\mu=\min\Big\{ 
\frac{1}{K}, 
L^{-\frac{1}{3}} \sqrt{\frac{\ln(N/\delta)}{K}} 
\Big\}, 
\beta_\calI= 
2\sqrt{\frac{ \ln(4T^2N/\delta) }{ \mu\abs{\calI}} } 
+ \frac{ \ln(4T^2N/\delta) }{ \mu\abs{\calI} }$\\
\nl {\bf Initialize}: $i=1$, $T_1=0$. \Comment{$i$ indexes an epoch}\label{line:rerun_beginning2}\\
\nl\label{line:adagreedy restart}\For(\Comment{$j$ indexes a block}){$j=1, 2, \ldots$}{ 
\nl Compute $\hat{\pi}_{(i,j)} = \argmax_{\pi\in\Pi} \avgR_{[T_i+1, T_i+2^{j-1}-1]}(\pi)$ 
\Comment{or arbitrary if $j=1$}\\
\nl \For{$t = T_i+2^{j-1}, \ldots, T_i+2^{j}-1$}{
\nl Set $p_t(a)=\mu + (1-K\mu) \one \{ a=\hat{\pi}_{(i,j)}(x_t) \}, 
\forall a\in [K]$ \\
\nl Play $a_t\sim p_t$ and receive $r_t(a_t)$ \\
\nl\label{line:trigger_rerun}\If{$(t\geq T_i+L) \text{ or } (j>1 \text{ and } \test(t)=\true)$}{
\nl                  $T_{i+1}\leftarrow t$, $i\leftarrow i+1$ \\ 
\nl                  \textbf{goto} Line~\ref{line:adagreedy restart}
              }   
          }
      }
\ \\
\textbf{Procedure\ }$\test(t)$\\
\nl $\ell=1$ \\
\nl \While{$\ell\leq t-T_i$}{
\nl    Let $A\triangleq [t-\ell+1, t]$ and $B\triangleq [T_i+1, T_i+2^{j-1}-1]$ \\ 
\nl\label{eqn:ada_compare1}\lIf{$ \avgR_{A}(\hat{\pi}_{A}) > \avgR_{A}(\hat{\pi}_B) + 2(\beta_{A} + \beta_{B}+ 2v)$}{
       \textbf{return} \true 
    }
\nl    $\ell \leftarrow 2\ell$
}
\nl \textbf{return} \false
\end{algorithm}


Note that whenever $v =O(L^{-\frac{1}{3}})$, the rate of the regret above is of order $\otil(L^{2/3})$ (since $|\calI| \leq L$), 
which matches the ordinary regret bound of \EG ($\otil(T^{2/3})$). 
While a condition on both the interval length and variation is seemingly strong
and the bound seems to be meaningful only for very small $v$,
we emphasize that 1) sublinear regret via oracle-efficient algorithms is impossible under a fully adversarial setting 
even for the classic regret~\citep{Hazan2016} and 2) based on Theorem~\ref{thm:AdaEG2} we can in fact derive strong dynamic regret bounds that hold without 
any assumption on the distribution sequence (see Section~\ref{sec:implications}).


\subsection{\AdaILTCB}
\begin{algorithm}[t]
\SetAlgoLined
\setcounter{AlgoLine}{0}
\SetAlgoVlined
\DontPrintSemicolon
\caption{\AdaILTCB}\label{alg:AdaILTCB2}
\nl {\bf Input}: largest allowed interval length $L$ and variation $v$, allowed failure probability $\delta$ \\ 
\nl {\bf Define}:
$\mu=\min\Big\{ \frac{1}{2K}, L^{-\frac{1}{2}}\sqrt{\frac{\ln(8T^2N^2/\delta)\ln(L)}{K}} \Big\}, C_1=4, C_2=10^6, C_3=1.1\times 10^3, C_4=41, C_5=1200, C_6=6.4$ \\
\nl {\bf Initialize}: $i=1$, $T_1=0$ \Comment{$i$ indexes an epoch}\label{line:rerun_beginning2}\\
\nl\label{line:adaILTCB restart}\For(\Comment{$j$ indexes a block}){$j=1,2,\ldots$} {
\nl       Let $Q_{(i,j)}$ be a solution to (OP) with parameter $\mu$ and data from $[T_i+1, T_i+2^{j-1}-1]$ \\
\nl       \For{$t = T_i+2^{j-1}, \ldots, T_i+2^{j}-1$}{
\nl          Set $p_t(a)=Q_t^{\mu}(a|x_t), \forall a\in [K]$ where  $Q_t = Q_{(i,j)}$\\
\nl          Play $a_t\sim p_t$ and receive $r_t(a_t)$ \\
\nl \label{line:trigger_rerun}  \If{$(t\geq T_i+L) \text{ or } (\test(t)=\true)$}{
\nl                  $T_{i+1}\leftarrow t$, $i\leftarrow i+1$ \\ 
\nl                  \textbf{goto} Line~\ref{line:adaILTCB restart}
              }   
          }
      }
\ \\
\textbf{Procedure\ }$\test(t)$\\
\nl $\ell=1$ \\
\nl \While{$\ell\leq t-T_i-1$}{
\nl   Let $A\triangleq [t-\ell,t-1]$ and $B\triangleq [T_i+1, T_i+2^{j-1}-1]$ \\
\nl\label{line:ILTCB:check_reg}\lIf{$\max_{\pi\in\Pi}\left\{\whReg_{B}(\pi)-C_1\whReg_{A}(\pi)\right\} > \frac{C_2LK\mu}{\ell} + C_3v$}{
       \textbf{return} \true
    }
\nl\label{line:ILTCB:check_reg_another}\lIf{$\max_{\pi\in\Pi}\left\{\whReg_{A}(\pi)-C_1\whReg_{B}(\pi)\right\} > \frac{C_2LK\mu}{\ell} + C_3v$}{
       \textbf{return} \true
    }
\nl\label{line:ILTCB:check_var}\lIf{$\max_{\pi\in\Pi}\left\{ \whV_{A}(Q_t,\pi) - C_4\whV_{B}(Q_t,\pi) \right\} > \frac{C_5LK}{\ell} + \frac{C_6 v}{\mu}$} {
       \textbf{return} \true
   }
   
\nl    $\ell \leftarrow 2\ell$
}
\nl \textbf{return} \false
\end{algorithm}

Although being fairly simple, \AdaEG is suboptimal just as \EG is
suboptimal for stationary environments.  In this section we propose
\AdaILTCB, a variant of \minimonster~\citep{AgarwalHsKaLaLiSc14},
which achieves the optimal regret rate while also being
oracle-efficient.  The idea is similar to \AdaEG, but the statistical
checks are more involved.  

For a policy $\pi$ and an interval $\calI$, we denote the expected and
empirical regret of $\pi$ by $\Reg_\calI(\pi) = \max_{\pi' \in
  \Pi}\calR_\calI(\pi') - \calR_\calI(\pi)$ and $\whReg_\calI(\pi) =
\max_{\pi' \in \Pi}\avgR_\calI(\pi')- \avgR_\calI(\pi)$ respectively.
For a context $x$ and a distribution over the policies $Q \in
\Delta^\Pi \coloneqq \{Q \in \fR^N_+: \sum_{\pi \in \Pi} Q(\pi) =
1\}$, the projected distribution over the actions is denoted by
$Q(\cdot|x)$ such that $Q(a|x) = \sum_{\pi: \pi(x)=a} Q(\pi), \;
\forall a \in [K]$. The smoothed projected distribution with a minimum
probability $\mu$ is defined as $Q^\mu(\cdot|x) = \mu\one +
(1-K\mu)Q(\cdot|x)$ where $\one$ is the all-one
vector. Like~\citep{AgarwalHsKaLaLiSc14}, we keep track of a bound
on the variance of the reward estimates and define for a policy $\pi$, an
interval $\calI$ and a distribution $Q \in \Delta^\Pi$
\[
\whV_\calI (Q, \pi) = \frac{1}{|\calI|} \sum_{t \in \calI} \left[ \frac{1}{Q^{\mu}(\pi(x_t) | x_t)} \right], 
\quad V_\calI(Q, \pi) = \frac{1}{|\calI|} \sum_{t \in \calI} \E_{x \sim D_t^\calX} \left[ \frac{1}{Q^{\mu}(\pi(x) | x)} \right].
\]

Similar to \AdaEG, the proposed algorithm \AdaILTCB (Algorithm~\ref{alg:AdaILTCB2}) proceeds like the base algorithm (\sloppy \minimonster in this case) with additional tests to detect the non-stationarity of the environment. We define an epoch and a block similar to those of \AdaEG. The algorithm solves the optimization (OP) defined in~\citep{AgarwalHsKaLaLiSc14} (and included in
Appendix~\ref{app:AdaILTCB2}) at the beginning of each block using the data collected in that epoch so far. 
The solution of (OP) is denoted by $Q_t$, a \textit{sparse} distribution over $\Pi$, and the learner samples actions based on $Q_t^{\mu}(\cdot|x_t)$. 

At each round, the \test checks whether the empirical regret or the variance of reward estimates of any policy has changed significantly in a recent interval (i.e., $[t-\ell,t-1]$), compared to the interval from which we compute $Q_t$ (i.e., $[T_i+1, T_i+2^{j-1}-1]$). If so, the algorithm restarts with a new epoch. Note that detecting the change of regret is similar to detecting the change of reward; but different from \AdaEG, here we also check the change of reward estimate variance. This inherits from the tighter variance control in \minimonster, the key to obtaining better regret compared to \EPG. 

\paragraph{Oracle-Efficiency.}
Note that Lines~\ref*{line:ILTCB:check_reg}, \ref*{line:ILTCB:check_reg_another} and \ref*{line:ILTCB:check_var}
can all be implemented by one call of
the AMO oracle each, after using two extra oracle calls to compute $\max_{\pi' \in \Pi}\avgR_{B}(\pi')$ and $\max_{\pi' \in \Pi}\avgR_{A}(\pi')$ in advance. 
Specifically, let $\calS = \{(x_\tau,
\frac{-1}{2^{j-1}-1}\ips_\tau)\}_{\tau\in B} \cup \{(x_\tau,
\frac{C_1}{\ell} \ips_\tau)\}_{\tau\in A}$, then the left hand
side of the inequality in Line~\ref*{line:ILTCB:check_reg} can be rewritten
as $\max_\pi \sum_{(x,r) \in \calS} r(\pi(x)) + \max_{\pi' \in
  \Pi}\avgR_{B}(\pi') - C_1\max_{\pi' \in
  \Pi}\avgR_{A}(\pi')$, where clearly the first term can be
computed by one oracle call and the rests are precomputed already.
Similarly, Line~\ref*{line:ILTCB:check_var} can be computed by feeding
the oracle with examples $ \{(x_\tau,
\frac{1}{\ell}\frac{1}{Q_\tau^{\mu}(\cdot| x_\tau)})\}_{\tau\in A} \cup
\{(x_\tau, \frac{-C_4}{2^{j-1}-1} \frac{1}{Q_\tau^{\mu}(\cdot |
  x_\tau)})\}_{\tau\in B}$. 
  
\citet{AgarwalHsKaLaLiSc14} showed that the optimization problem (OP) can be
solved by $\otil(1/\mu)$ oracle calls and the solution
has only $\otil(1/\mu)$ non-zero coordinates. Note that we only solve (OP) at the beginning of each block. Since there are $\order(\ln L)$ blocks in an epoch, the total oracle calls in an epoch is bounded by $\otil\big(\ln(L)/\mu \big) = \otil(\sqrt{LK})$, which amortizes to $\otil(S^\p\sqrt{LK}/T)$ per round if there are $S^\p$ epochs
(in Section~\ref{sec:implications} we relate $S'$ to $S$ or $\bvar$). 

We next present the interval regret guarantee of \AdaILTCB,
which improves from $\otil(L^\frac{2}{3})$ to $\otil(\sqrt{L})$ compared to \AdaEG
(see Appendix~\ref{app:AdaILTCB2} for the proof),
except that it holds for interval with total variation $\bvar_\calI$ (instead of reward variation $\var_\calI$) bounded by $v$
due to the fact that variation in the context is important for the variance control (Line \ref*{line:ILTCB:check_var}).

\begin{theorem}\label{thm:AdaILTCB2}
With probability at least $1 - \delta$, for any interval $\calI$ such
that $|\calI| \leq L$ and $\bar{\var}_\calI\leq v$, 
\AdaILTCB with parameters $L$, $v$, and $\delta$
guarantees for any $\pi \in \Pi$,
\[
\sum_{t \in \calI} r_t(\pi(x_t)) - r_t(a_t) \leq
\otil\left(\abs{\calI}v + \sqrt{LK\ln(N/\delta)}\right). 
\]
\end{theorem}

\section{Implications}\label{sec:implications}
In this section we discuss the implications of interval regret
guarantees on switching/dynamic regret, both of which are
meaningful performance measures for non-stationary environments.

\paragraph{Switching Regret.} We begin with switching regret, which is pretty straightforward.
One only needs to divide the entire time interval $[1,T]$ into several i.i.d. subintervals with length bounded by $L$,
and then apply the interval regret guarantee on each of these subintervals since the best policy $\pi^\star_t$ remains the same on each of these subintervals
(for \AdaEG and \AdaILTCB we can simply set the variation tolerance $v$ to be $0$).
We take Exp4.S as an example and state the results below 
(see Appendix~\ref{app:Exp4.S} for the proof),
while similar results for other algorithms are summarized in Table~\ref{tab:results}.

\begin{cor}\label{cor:Exp4.S} 
Exp4.S with parameter $L$ ensures $\E\left[ \sum_{t=1}^T r_t(\pi_t^\star(x_t)) - r_t(a_t) \right] \leq 
\otil\left(\left(\frac{T}{\sqrt{L}} + S\sqrt{L}\right) \sqrt{K\ln N}\right)$ where $S = 1 + \sum_{t=2}^T \one\{\calD_t \neq \calD_{t-1}\}$.
\end{cor}

If $S$ is known, then setting $L = T/S$ gives a bound of $\otil(\sqrt{STK\ln N})$.
Otherwise setting $L$ with different values leads to different bounds that are incomparable.
For example, setting $L = T$ leads to $\otil(S\sqrt{TK\ln N})$ while setting 
$L = \sqrt{T}$ leads to $\otil((T^\frac{3}{4} + ST^\frac{1}{4})\sqrt{K\ln N})$.
No matter how $L$ is tuned, however,
these bounds all become vacuous ($\Omega(T)$) when $S$ is large enough but still sublinear in $T$,
an issue addressed later in Section~\ref{sec:bin}.

\paragraph{Dynamic Regret.}
We now move on to discuss dynamic regret in terms of the variation measures $\var$ or $\bvar$ (recall $\var = \sum_{t=2}^T \max_{\pi \in \Pi} |\calR_t(\pi) - \calR_{t-1}(\pi)|$ and $\bvar = \sum_{t=2}^T \TVD{\calD_t-\calD_{t-1}}$). 
We first point out that previous
works~\citep{BesbesGuZe15, ZhangYaJiZh17} have studied a reduction
from dynamic regret to interval regret, restated below:

\begin{lemma}\label{lem:dynamic2interval}
Let $\{\calI_i = [s_i, t_i]\}_{i \in [n]}$ be time intervals that partition $[1,T]$. 
We have 
\[
\sum_{t=1}^T \E_t\left[ r_t(\pi^\star_t(x_t)) - r_t(a_t) \right] \leq
\sum_{i=1}^n \sum_{t\in\calI_i} \E_t\left[ \left(
  r_t(\pi^\star_{s_i}(x_t)) - r_t(a_t) \right)\right] +
2\sum_{i\in[n]} |\calI_i|\var_{\calI_i}.
\]
\end{lemma}

We include the proof in Appendix~\ref{app:dynamic2interval} for
completeness.  Partitioning $[1,T]$ into intervals with equal length $L'
\leq L$, applying this lemma and Theorem~\ref{thm:Exp4.S},
and using the fact $\sum_{i\in[n]} \var_{\calI_i} \leq \var$ directly
lead to the following result for Exp4.S.

\begin{cor}\label{thm:Exp4.S_dynamic}
Exp4.S with parameter $L$ ensures that $\E\left[ \sum_{t=1}^T
  r_t(\pi^\star_t(x_t)) - r_t(a_t) \right] \leq \otil\left(\min_{0\leq
  L' \leq L} \left\{\frac{T}{L'}\sqrt{LK\ln N} + L'\var\right\} \right)$.
\end{cor}

Again, if $\var$ is known one can tune $L$
optimally to get a bound $\otil(T^\frac{2}{3}(\Delta K\ln
N)^\frac{1}{3} + \sqrt{TK\ln N})$, similar to the optimal
dynamic regret in multi-armed bandits~\citep{BesbesGuZe14}.  When
$\Delta$ is unknown, different values of $L$ give different and in
general incomparable bounds.  For example, setting $L =
T^\frac{2}{3}$ leads to $\otil(T^\frac{2}{3}\sqrt{\var}(K\ln
N)^\frac{1}{4} +T^\frac{2}{3}\sqrt{K\ln N})$
(with $L' = \min\{T^{2/3}, T^{2/3}(K\ln N)^{1/4}/\sqrt{\var}\}$ in this case),
which is again vacuous for large $\var$.

Similar arguments also provide a dynamic regret bound for \corral with \bistro in the transductive setting,
as shown in Table~\ref{tab:results} (also see Corollary~\ref{cor:BISTRO+_dynamic} in Appendix~\ref{app:BISTRO+}).
However, the exact same argument above does not apply to \AdaEG and \AdaILTCB
directly since its interval regret guarantee requires $\var_\calI \leq
v$.  
It turns out, however, one can set $v$ to some carefully selected value and
partition $[1,T]$ correspondingly so that every subinterval satisfies $|\calI| \leq L$
and $\var_\calI \leq v$,
to obtain the following results
that hold in a completely adversarial setting.\footnote{The dependence on
  $K\ln(N/\delta)$ in these results is slightly loose for conciseness and could be tightened.}
  The proofs are included in Appendix~\ref{app:dynamic2interval}.

\begin{cor}
\label{cor:dynamic AdaEG}
With probability at least $1 - \delta$, \AdaEG with parameter $L$, $\delta$ and
$v = L^{-1/3}$ ensures that
\[
\sum_{t=1}^T r_t(\pi_t^\star(x_t)) - r_t(a_t) \leq
\otil\left(\left(  \frac{T}{L^{1/3}} + L^{1/3}\sqrt{\var T}
\right)K\ln(N/\delta)\right).
\]
Specifically, if $\var$ is known, setting $L =
\min\{(T/\var)^\frac{3}{4}, T\}$ gives
$\otil((\var^\frac{1}{4}T^\frac{3}{4} +
T^\frac{2}{3})K\ln(N/\delta))$; otherwise, setting $L = T^\frac{3}{4}$
gives $\otil((\sqrt{\var}+1)T^\frac{3}{4} K\ln(N/\delta))$.
\end{cor}


\begin{cor}
\label{cor:dynamic AdaILTCB}
With probability at least $1 - \delta$, \AdaILTCB with parameter $L$, $\delta$ and
$v = L^{-\frac{1}{2}} $ ensures that
\[
\sum_{t=1}^T r_t(\pi_t^\star(x_t)) - r_t(a_t) \leq
\otil\left(\left(\frac{T}{\sqrt{L}} + \bvar L
\right)K\ln(N/\delta)\right).
\]
If $\bvar$ is known, setting $L =
\min\{(T/\bvar)^\frac{2}{3}, T\}$ gives
$\otil((\bvar^\frac{1}{3} T^\frac{2}{3} +\sqrt{T})K\ln(N/\delta))$; otherwise, setting $L = T^\frac{2}{3}$
gives $\otil((\bvar+1) T^\frac{2}{3} K\ln(N/\delta))$.
\end{cor}

One can see that again the result for \AdaILTCB is better than that of \AdaEG,
and is in fact very close to that of the inefficient baseline Exp4.S,
except that it is in terms of the slightly larger variation measure $\bvar$.

\section{Achieving Switching/Dynamic Regret with No Parameters}
\label{sec:bin}

\begin{algorithm}[t]
\SetAlgoLined
\setcounter{AlgoLine}{0}
\SetAlgoVlined
\DontPrintSemicolon
\caption{\AdaBIN}\label{alg:AdaGreedy.bin}
\nl {\bf Input}: allowed failure probability $\delta$ \\ 
\nl {\bf Define}:
$\beta_\calI= 2\sqrt{\frac{\ln(4T^2N/\delta)}{\mu_\calI \abs{\calI}}}+\frac{\ln(4T^2N/\delta)}{\mu_\calI \abs{\calI}}$, where $\displaystyle\mu_\calI\triangleq \min_{t\in\calI} \mu_t$ and $\mu_t$ is defined below, $\alpha_\calI = 2\sqrt{\frac{K\ln(4T^2N/\delta)}{\abs{\calI}}}+\frac{K\ln(4T^2N/\delta)}{\abs{\calI}}$\\
\nl {\bf Initialize}: $t=1$, $i=1$, $T_1=0$ \Comment{$i$ indexes an epoch}\label{line:rerun_beginning2}\\
\nl\label{line:adagreedy restart}\For(\Comment{$j$ indexes a block}){$j=1, 2, \ldots$}{ 
\nl Compute $\hat{\pi}_{(i,j)} = \argmax_{\pi\in\Pi}\avgR_{[T_i+1, T_i+2^{j-1}-1]}(\pi)$ 
\Comment{or arbitrary if $j=1$}\\
\nl $H=2^{j-1}$ \Comment{$H$ is block length}\\
\nl\label{line:forloop of bin}\For(\Comment{$b$ indexes a bin, each with length $\sqrt{H}$}){$b = 1, 2, \ldots, \sqrt{H}$}{
\nl\label{line:exploration prob}Make bin $b$ an exploration bin with probability $1/\sqrt{b}$; otherwise an exploitation bin\\
\nl\label{line:forloop of bin step}\For(\Comment{loop through rounds in bin $b$}){$\tau = 1, \ldots, \sqrt{H}$}{
\nl       Let $\mu_t=\min\Big\{ \frac{1}{K}, (t-T_i)^{-\frac{1}{3}}\sqrt{\frac{\ln(N/\delta)}{K}} \Big\}$\\
\nl       Set $p_t(a)=\begin{cases}
             \frac{1}{K}, &\text{if bin $b$ is an exploration bin}, \\
             \mu_t+(1-K\mu_t)\one \{ a=\hat{\pi}_{(i,j)}(x_t) \}, &\text{if bin $b$ is an exploitation bin}.
          \end{cases}$ \\
\nl       Play $a_t\sim p_t$ and receive $r_t(a_t)$ \\
\nl\label{line:trigger_rerun_ada3}  \If{$j>1$ and (bin $b$ is exploration bin) and ($\test(t)=\true$)}{
\nl                  $T_{i+1}\leftarrow t$, $t\leftarrow t+1$, $i\leftarrow i+1$ \\ 
\nl                  \textbf{goto} Line~\ref{line:adagreedy restart}
              }
\nl           $t\leftarrow t+1$
          }
      }
      }
\ \\
\textbf{Procedure\ }$\test(t)$\\
\nl $\ell=1$ \\
\nl \While{$[t-\ell+1, t]$ is a subset of the current bin}{
\nl    Let $A\triangleq [t-\ell+1,t]$ and $B\triangleq [T_i+1, T_i+2^{j-1}-1]$\\
\nl\label{eqn:ada_compare2}\lIf{$ \avgR_{A}(\hat{\pi}_{A}) > \avgR_{A}(\hat{\pi}_B) + 2(\alpha_{A}+ \beta_{B})$}{
      \textbf{return} \true
    }
\nl    $\ell \leftarrow 2\ell$
}
\nl \textbf{return} \false
\end{algorithm}


As mentioned, when the parameter $S$ or $\var$ is unknown, our algorithms achieve regret of the form
$\otil(S^{c_1} T^{c_2})$ or $\otil(\var^{c_1} T^{c_2})$ for some exponents $c_1$ and $c_2$ such that $c_1+c_2>1$,
which is vacuous when $S$ or $\var$ is large.
The hope here is to obtain a bound with $c_1+c_2=1$ as in the case when the parameters are known.
Observe that if an algorithm was able to achieve interval regret $o(|\calI|)$ simultaneously for all intervals $\calI$,
which is called strongly adaptive algorithm~\citep{DanielyGoSh15},
then by similar reductions discussed in Section~\ref{sec:implications} one could derive switching/dynamic regret with $c_1+c_2=1$.
However, it was shown by~\citet{DanielyGoSh15} that a strongly adaptive algorithm is impossible for the bandit setting.

Despite this negative result, \citet{KarninAn16} developed new techniques and proposed a parameter-free algorithm for the two-armed bandit setting 
with dynamic regret $\otil(\var^{0.18}T^{0.82})$.
While their algorithm and analysis do not directly generalize to the multi-armed or contextual setting,
here we extract their idea of \textit{bin-based exploration} and incorporate it into our \AdaEG algorithm,
leading to a parameter-free algorithm called \AdaBIN with regret $\otil(\min\{S^{\frac{1}{4}}T^{\frac{3}{4}}, \var^{\frac{1}{5}}T^{\frac{4}{5}}\})$.
This improves and generalizes the result of~\citet{KarninAn16} significantly.



Similar to \AdaEG, \AdaBIN computes the empirical best policy at the beginning of each block, and plays it throughout that block, except for some exploration steps. 
The differences are 1) each block is further divided into {\it bins} with equal length; 
2) in addition to the small probability of exploration $\mu_t$ at each round, some bins are randomly selected for pure exploration;
3) the non-stationarity test is only executed in exploration bins, and only checks for intervals within the bin;
4) parameters $L$ and $v$ are removed and the exploration probability $\mu_t$ is set adaptively.
Clearly \AdaBIN is still oracle-efficient.

Comparing the non-stationarity tests of \AdaEG and \AdaBIN, 
one can see that the term $\beta_A = \otil(1/\sqrt{\mu\ell})$ in the former is replaced by the term $\alpha_A = \otil(\sqrt{K/\ell})$ in the latter.
This is due to the lower variance of reward estimates from the pure exploration bin and plays a crucial role in our analysis to achieve the following bound.



\begin{theorem} 
\label{thm:Ada3 regret}
With probability at least $1-6\delta$, \AdaBIN\ with parameter $\delta$ guarantees
\begin{align*}
\sum_{t=1}^T r_t(\pi_t^\star(x_t))-r_t(a_t) \leq \otil\left(K\ln(N/\delta) \min\left\{S^{\frac{1}{4}}T^{\frac{3}{4}}, \var^{\frac{1}{5}}T^{\frac{4}{5}}+T^{\frac{3}{4}} \right\} \right).
\end{align*}
\end{theorem}

This bound is sublinear as long as $S$ or $\var$ is sublinear, 
and is stronger than those in the ``param-free'' column of Table~\ref{tab:results} if $S = \Omega(T^{1/3})$ or $\var = \Omega(T^{4/9})$ (but still sublinear).
One might wonder whether combining the bin-based exploration idea with \AdaILTCB leads to even better results.
The answer is unfortunately no because the dominant part of the regret is not from the $\epsilon$-greedy part of the algorithm but the bin explorations.
We leave the question of whether better results of this kind are possible as a future direction.

Due to the existence of exploration bins, \AdaBIN can have poor regret on some intervals. In fact, we can only show a loose $\otil({T^{\frac{3}{4}}})$ interval regret bound for this algorithm as shown in Table~\ref{tab:results}. For completeness, we provide a proof in Appendix~\ref{appendix:interval_adabin}. 

\section{Conclusions}
In this work we take the first step in studying the problem of non-stationary contextual bandit. 
We propose several new algorithms and provide a number of achievable results under various regret notions.
More future directions include 1) deriving algorithms with long term memory so as to identify distributions experienced before~\citep{bousquet2002tracking};
2) designing simpler and more practical algorithms, given that our current methods have several impractical aspects such as restarting.

\paragraph{Acknowledgement.}
CYW is grateful for the support of NSF Grant \#1755781.

\newpage
\bibliography{ref}
\newpage
\appendix

\section{Preliminaries}\label{app:Freedman}

Our analysis relies on the following Freedman's inequality.
\begin{lemma}[\citep{BeygelzimerLaLiReSc11}]
\label{lem:freedman}
Let $X_1, \ldots, X_n \in \fR$ be a sequence of random variables such that 
$X_i \leq R$ and $\E[X_i | X_{i-1}, \ldots, X_1] = 0$ for all $i \in [n]$.
Then for any $\delta \in (0,1)$ and $\lambda \in [0, 1/R]$, with probability at least $1-\delta$, we have
\[
\sum_{i=1}^n X_i \leq (e-2)\lambda V + \frac{\ln(1/\delta)}{\lambda}
\]
where $V = \sum_{i=1}^n \E[X_i^2 | X_{i-1}, \ldots, X_1]$.
\end{lemma}


The following lemmas relates the two variation notions we use.
\begin{lemma}
\label{lemma:variation_relation}
  For any interval $\calI$, $\var_\calI \leq \bvar_\calI$. 
\end{lemma}

\begin{proof}
  Let $\pi$ be any policy, 
  \begin{align*} 
    |\calR_t(\pi) - \calR_{t-1}(\pi)| 
    & = \bigabs{ \E_{(x,r)\sim \calD_t}[r(\pi(x))] - \E_{(x,r)\sim \calD_{t-1}}[r(\pi(x))] }\\
    & = \left| \int_{[0,1]^K}\int_{\calX} (\calD_{t}(x,r)-\calD_{t-1}(x,r)) r(\pi(x))dxdr \right| \\
    & \leq \int_{[0,1]^K}\int_{\calX} \bigabs{\calD_{t}(x,r)-\calD_{t-1}(x,r)} dxdr \\
    & = \TVD{\calD_{t}-\calD_{t-1}}. 
  \end{align*}
  Thus, $\max_{\pi\in\Pi}|\calR_t(\pi) - \calR_{t-1}(\pi)| \leq \TVD{\calD_{t}-\calD_{t-1}}$. Summing over $\calI$ gives $\var_\calI\leq \bvar_\calI$. 
\end{proof}

\section{Exp4.S Algorithm and Proofs}\label{app:Exp4.S}

\LinesNumberedHidden
\begin{algorithm}[H]
\DontPrintSemicolon
\caption{Exp4.S}\label{alg:Exp4.S}
{\bf Input}: largest interval length of interest $L$ \\ 
Define $\eta = \sqrt{\nicefrac{\ln(NL)}{LK}}$ and $\mu = \nicefrac{1}{NL}$ \\
Initialize $P_t \in \Delta^\Pi$ to be the uniform distribution over policies. \\
\For{$t=1, \ldots, T$} {
     see $x_t$, play $a_t \sim p_t$ where $p_t(a) = \sum_{\pi: \pi(x_t) = a} P_t(\pi), \;\forall a\in[K]$ \\
     receive $r_t(a_t)$ and construct $\wh{c}_t(a) = \frac{1 - r_t(a)}{p_t(a)}\one\{a = a_t\}, \;\forall a\in[K]$ \\
     set $\tilde{P}_{t+1}(\pi) \propto P_t(\pi)\exp(-\eta \wh{c}_t(\pi(x_t))), \;\forall \pi \in \Pi$ \\
     set $P_{t+1}(\pi) = (1-N\mu)\tilde{P}_{t+1}(\pi) + \mu,  \;\forall \pi \in \Pi$ \\
}
\end{algorithm}

The Exp4.S algorithm is presented in Algorithm~\ref{alg:Exp4.S},
which is a direct generalization of Exp3.S~\citep{AuerCeFrSc02}.
Note that we use loss estimates $\wh{c}_t$ instead of reward estimate $\ips_t$
in the multiplicative update, and naturally we define $c_t = \one - r_t$.

\begin{proof}[Proof of Theorem~\ref{thm:Exp4.S}]
Using the fact $e^{-y} \leq 1-y+y^2$ for any $y \geq 0$, $\ln(1+y) \leq y$ and $c_t(a) \in [0,1]$, we have
\begin{align*}
&\ln \left( \sum_{\pi' \in \Pi} P_t(\pi') \exp( -\eta \wh{c}_t(\pi'(x_t)) ) \right) 
\leq \ln \left( \sum_{\pi' \in \Pi} P_t(\pi') (1 -\eta \wh{c}_t(\pi'(x_t)) + \eta^2 \wh{c}_t(\pi'(x_t))^2 \right)\\
&= \ln \left(1 - \eta {c}_t(a_t) + \eta^2 \wh{c}_t(a_t) c_t(a_t) \right) 
\leq  - \eta {c}_t(a_t)  + \eta^2 \wh{c}_t(a_t). 
\end{align*}
On the other hand, we have for any fixed $\pi$, 
\begin{align*}
\ln \left( \sum_{\pi' \in \Pi} P_t(\pi') \exp( -\eta \wh{c}_t(\pi'(x_t)) ) \right) 
&= \ln\left( \frac{P_t(\pi) \exp( -\eta \wh{c}_t(\pi(x_t)) )}{\tilde{P}_{t+1}(\pi)}  \right) \\
&= \ln\left( \frac{P_t(\pi)(1-N\mu)  }{P_{t+1}(\pi)-\mu} \right)  -\eta \wh{c}_t(\pi(x_t))  \\
&\geq \ln(1-N\mu) + \ln\left( \frac{{P}_t(\pi) }{{P}_{t+1}(\pi)} \right)  -\eta \wh{c}_t(\pi(x_t))  \\
&\geq -2N\mu + \ln\left( \frac{{P}_t(\pi) }{P_{t+1}(\pi)} \right)  -\eta \wh{c}_t(\pi(x_t))  
\end{align*}
where the last step is by the fact $N\mu \leq \frac{1}{2}$ and thus
$\ln(\frac{1}{1-N\mu}) = \ln(1 + \frac{N\mu}{1-N\mu}) \leq \ln(1 + 2N\mu) \leq 2N\mu$.
Combining the above two displayed equations, summing over $t\in\calI$, telescoping and rearranging gives
\[
\sum_{t\in\calI} c_t(a_t) - \wh{c}_t(\pi(x_t))  \leq 
\frac{\ln(1/\mu) + 2LN\mu}{\eta} + \eta \sum_{t\in\calI}\wh{c}_t(a_t).
\]
Taking the expectation on both sides, using the fact $\E_{a_t\sim p_t}[\wh{c}_t(a_t)] \leq K$,
and plugging $c_t(a) = 1 - r_t(a)$, $\eta$ and $\mu$ finish the proof.
\end{proof}

\begin{proof}[Proof of Corollary~\ref{cor:Exp4.S}]
We first partition $[1, T]$ evenly into $T/L$ intervals,
then within each interval, further partition it into several subintervals
so that $\calD_t$ remains the same on each subinterval.
Since the number of switches is at most $S-1$, 
this process results in at most $T/L + S$ subintervals, each with length at most $L$.
We can now apply Theorem~\ref{thm:Exp4.S} to each subinterval and
sum up the regrets to get the claim bounds.
\end{proof}


\section{Proofs for \AdaEG}\label{app:AdaEG2}
Before we prove the theorems, we first define some notations that facilitate the analysis. These notations are used throughout Appendix~\ref{app:AdaEG2},~\ref{app:AdaILTCB2}, and~\ref{app:AdaBIN}. In \AdaEG and \AdaILTCB, we define $\flag_t=(t\geq T_i+L) \text{ or } (j>1 \text{ and }\test(t)=\true)$, where $i$ and $j$ are the epoch and block indices $t$ is in. 
This is exactly the condition that triggers the rerun of the algorithm. In all three algorithms, we let $B(i,j)\triangleq [T_i+1, T_i+2^{j-1}-1]$; sometimes when $i, j$ are already specified, we simply write $B$. Note that when $j=1$, $B=[T_i+1, T_i]$, which is an empty set. In this case, instead of defining $\beta_B$ (which is used in \AdaEG and \AdaBIN) to be infinity, we let it to be zero. This just makes some analysis easier. 

Below we state a few useful lemmas before proving the main theorem.
\begin{lemma}
  For any interval $\calI$ such that $\Delta_\calI \leq v$, we have for
  any sub-intervals $\calI_1, \calI_2 \subseteq \calI$ and any $\pi \in \Pi$, 
  \[
  |\calR_{\calI_1}(\pi) - \calR_{\calI_2}(\pi)| \leq v. 
  \]
  \label{lemma:var-bound}
\end{lemma}

\begin{proof}
  The proof involves noticing for that any two rounds $s,t \in \calI$
  and $\pi \in \Pi$, $|\calR_s(\pi) - \calR_t(\pi)| \leq
  v$. This is easily seen using triangle
  inequality, since assuming $s < t$, 

  \begin{align*}
    |\calR_s(\pi) - \calR_t(\pi)| &\leq \sum_{\tau=s+1}^t
    |\calR_{\tau}(\pi) - \calR_{\tau-1}(\pi)| \leq \sum_{\tau \in
    \calI} |\calR_{\tau}(\pi) - \calR_{\tau-1}(\pi)| \leq v. 
  \end{align*}

  The lemma is now immediate, since 

  \begin{align*}
    |\calR_{\calI_1}(\pi) - \calR_{\calI_2}(\pi)| &\leq
    \frac{1}{|\calI_1|}\, \frac{1}{|\calI_2|}\, \sum_{s \in \calI_1}
    \sum_{t \in \calI_2} |\calR_s(\pi) - \calR_t(\pi)| \leq v.
  \end{align*}
\end{proof}

\begin{definition}[$\evt 1$]
Define $\evt 1$ to be the following event: for all $\calI\subseteq [1,T]$ and all $\pi\in\Pi$, 
\begin{align}
\left\vert \avgR_{\calI}(\pi)-\calR_{\calI}(\pi) \right\vert \leq \beta_\calI. \label{eq:concentration2} 
\end{align}
\end{definition}
Recall $\ips_t(a) = \frac{r_t(a)}{p_t(a)} \one\{a = a_t\} \leq 1/\mu$ and
$\E_t[\ips_t(\pi(x_t))] = \calR_t(\pi),  \E_t[\ips_t(\pi(x_t))^2] \leq 1/\mu$.
By Freedman's inequality (Lemma~\ref{lem:freedman}) and a union bound, 
we have with probability at least $1-\delta/2$, $\evt 1$ holds. 

\begin{lemma}
\label{lemma:at most one rerun}
Consider an interval $\calI$ where $\abs{\calI}\leq L$ and $\Delta_\calI \leq v$. If $\evt 1$ holds, then there is at most one $t\in \calI$ such that $\flag_t=\true$ ($\flag_t$ is defined at the beginning of Appendix~\ref{app:AdaEG2}).
\end{lemma}
\begin{proof}
Let there be multiple such time instances. Let $t^\p, t\in \calI$ be two consecutive ones and $t^\p<t$. 
Note that $\flag_t=\true$ has two possible cases: $t\geq t^\p+L$ or $(j>1 \text{ and }\test(t)=\true)$. 
The former case cannot happen because $\abs{\calI}\leq L$. Now assume the latter. Let $i,j$ be the epoch and block index at time $t$ respectively. Define $A=[t-\ell+1,t]$ to be the interval that makes $\test(t)$ return \true, and define $B=[t^\p+1, t^\p+2^{j-1}-1]$. Then $\test(t)=\true$ implies 
\begin{align}
\avgR_{A}(\hat{\pi}_{A}) > \avgR_{A}(\hat{\pi}_B) + 2\beta_{A} + 2\beta_{B} + 4v. \label{eqn:test fail}
\end{align}
By the optimality of $\hat{\pi}_B$, We have 
\begin{align}
\avgR_{B}(\hat{\pi}_{A}) \leq \avgR_{B}(\hat{\pi}_B). \label{eqn:optimality of pi t}
\end{align}
Combining Eq.~\eqref{eqn:test fail} and Eq.~\eqref{eqn:optimality of pi t}, we see that either $\pi=\hat{\pi}_{A}$ or $\pi=\hat{\pi}_{B}$ will make the following inequality hold: 
\begin{align}
\bigabs{\avgR_{A}(\pi)-\avgR_{B}(\pi)} > \beta_{A} +\beta_{B} + 2v. \label{eqn:one side larger}
\end{align}
Thus, 
\begin{align*}
\bigabs{\calR_{A}(\pi)-\calR_{B}(\pi)}&\geq \bigabs{\avgR_{A}(\pi) - \avgR_{B}(\pi)} - \beta_{A} - \beta_{B}  \tag{by \eqref{eq:concentration2}} \\
& > \beta_{A} + \beta_{B} + 2v - \beta_{A}- \beta_{B} \geq 2v. \tag{by \eqref{eqn:one side larger}}
\end{align*}
On the other hand, by Lemma~\ref{lemma:var-bound}, we actually have $\bigabs{\calR_{A}(\pi)-\calR_{B}(\pi)} \leq v$, which leads to a contradiction. Thus we can conclude that such $t$ does not exist. 

\end{proof}


\begin{proof}[Proof of Theorem~\ref{thm:AdaEG2}]
We condition on $\evt 1$. When this event holds true, by Lemma~\ref{lemma:at most one rerun}, there is at most one $t\in \calI$ such that $\flag_t$ is \true (that is, rerun triggered at $t$). 
With this fact, we can focus on the case in which $\flag_t$ is \false for all $t\in\calI$. If there is actually a $t^\p\in\calI$ such that $\flag_{t^\p}=\textit{True}$, 
we can divide $\calI$ into $\calI_1 \cup \{t^\p\} \cup \calI_2$ and bound the regret in $\calI_1$ and $\calI_2$ separately. 
The total regret on $\calI$ would then be bounded by their sum plus $1$, which is still of the same order.  

We will also use the fact (proven in Lemma~\ref{lemma:var-bound})
that by the condition $\var_\calI \leq v$, we have
for any $\calI_1, \calI_2 \subset \calI$ and $\pi \in \Pi$,
\begin{equation}\label{eq:difference_in_R2}
|\calR_{\calI_1}(\pi) - \calR_{\calI_2}(\pi)|\leq v.
\end{equation}

Let $\calI=[s,e]$. For any $t\in \calI$, define $\ell_t=2^{\floor{\log_2(t-s+1)}}$ (i.e., $\ell_t$ is the longest $\ell\in\{1,2,4,8,\ldots\}$ such that $[t-\ell+1, t]\subseteq\calI$). Now focus on a specific $t$ that is in epoch $i$ and block $j$. Denote $A=[t-\ell_t+1,t], B=[T_i+1, T_i+2^{j-1}-1]$. Assuming the case described above (i.e., for all $\tau\in\calI$, $\flag_\tau=\false$), we have for $j>1$ and  any $\pi\in \Pi$, 
\begin{align*}
\calR_t(\pi)&\leq \calR_{A}(\pi) + v 
\leq \avgR_{A}(\pi) + \beta_{A} + v \tag{by \eqref{eq:difference_in_R2} and \eqref{eq:concentration2}}\\
&\leq \avgR_{A}(\hat{\pi}_{A}) + \beta_{A} + v \tag{by the optimality of $\hat{\pi}_{A}$}\\
&\leq \avgR_{A}(\hat{\pi}_{B}) + 3\beta_{A} + 2\beta_{B} + 5v \tag{$\test(t)=\false$}\\
&\leq \calR_{A}(\hat{\pi}_{B}) + 4\beta_{A} + 2\beta_B + 5v 
\leq \calR_{t}(\hat{\pi}_{B}) + 4\beta_{A} + 2\beta_B + 6v. \tag{by \eqref{eq:concentration2} and \eqref{eq:difference_in_R2}}
\end{align*}
Note that $\beta_B=\order(\beta_A)$ because $ \abs{B}=2^{j-1}-1\geq \frac{2^j-1}{3} \geq \frac{t-T_i}{3} \geq \frac{\ell_t}{3} = \frac{\abs{A}}{3}$. 

Now using this bound for all $t\in \calI$ and noting that there is at most one round with $j=1$, we can bound the sum of conditional expected regrets by 
\begin{align*}
\sum_{t \in \calI} \E_t[ r_t(\pi(x_t)) - r_t(a_t)]
&\leq {\sum_{t\in\calI} (\calR_t(\pi_t^\star)-\calR_t(\hat{\pi}_B) + K\mu)}\\
&\leq \order\left({\sum_{t\in\calI} (v + \beta_{[t-\ell_t+1,t]} + K\mu)}\right)\\
&=\tilde{\mathcal{O}}\left(\abs{\calI}v + \sum_{t\in\calI} \left(\sqrt{\frac{\ln(N/\delta)}{\mu\ell_t}} + \frac{\ln(N/\delta)}{\mu\ell_t} + K\mu\right)\right) \\
&=\tilde{\mathcal{O}}\left( \abs{\calI}v + L^{\frac{1}{6}}\sqrt{K\abs{\calI}\ln(N/\delta)} + L^{\frac{1}{3}}\sqrt{K\ln(N/\delta)} \right),
\end{align*}
where in the last step we use the fact $|\calI|L^{-\frac{1}{3}} \leq L^{\frac{1}{6}}\sqrt{|\calI|}$ for $|\calI| \leq L$.
Finally, applying Hoeffding-Azuma inequality finishes the proof.
\end{proof}

\section{Omitted Details for \AdaILTCB}\label{app:AdaILTCB2}

\begin{figure}[h]
\begin{framed}
\begin{center}
{\bf Optimization Problem} (OP)
\end{center}

Given a time interval $\calI$ and minimum probability $\mu$, find $Q \in \Delta^\Pi$ such that
for constant $B = 5 \times 10^5 $:
\begin{equation}\label{eq:low_empirical_regret2}
\sum_{\pi \in \Pi} Q(\pi) \whReg_\calI(\pi) \leq 2B K \mu
\end{equation}
\begin{equation}\label{eq:low_empirical_variance2}
\forall \pi \in \Pi: ~\whV_\calI(Q, \pi) \leq 2 K + \frac{\whReg_\calI(\pi)}{B\mu}
\end{equation}
\end{framed}
\caption{A subroutine for \AdaILTCB, adapted from~\citep{AgarwalHsKaLaLiSc14}}
\label{fig:OP}
\end{figure}

The optimization problem (OP) needed for \AdaILTCB is included in Figure~\ref{fig:OP}.
It is almost identical to the one proposed in~\citep{AgarwalHsKaLaLiSc14} except:
1) Instead of returning a sub-distribution, our version returns an exact distribution.
However, as discussed in~\citep{AgarwalHsKaLaLiSc14} this makes no real difference
since given a sub-distribution which satisfies Eq.~\eqref{eq:low_empirical_regret2}
and Eq.~\eqref{eq:low_empirical_variance2}, one can always put all the remaining 
weight on the empirical best policy $\argmax_\pi \avgR_\calI(\pi)$ to obtain a distribution 
that still satisfies those two constraints.
2) The constant $B$ used in~\citep{AgarwalHsKaLaLiSc14} is $100$.
It is also clear from the proof of~\citep{AgarwalHsKaLaLiSc14} that the value of this constant 
does not affect the feasibility of (OP) nor the efficiency of finding the solution.

Let $d = \ln(8T^2N^2/\delta)\ln(L)$. 
Without loss of generality, below we assume $L \geq 4Kd$ so that 
$\mu = \min\{\frac{1}{2K}, \sqrt{\frac{d}{KL}}\} = \sqrt{\frac{d}{LK}}$.
Indeed, if $L < 4Kd$, then the bound in Theorem~\ref{thm:AdaILTCB2} holds trivially since $L \leq 2\sqrt{LKd}$.
The fact $d/\mu  = LK\mu$ will be used frequently.
We use $V_t$ as a shorthand for $V_{\{t\}}$, that is, $V_t(Q, \pi) = \E_{x \sim D_t^\calX} \left[ \frac{1}{Q^{\mu}(\pi(x) | x)} \right]$.

We first state two lemmas that relates the variation of $\Reg_t(\pi)$ and $V_t(Q,\pi)$ to $\bvar$, and then two lemmas on the concentration bounds of empirical reward and empirical variance.

\begin{lemma}\label{lemma:variation and regret}
For any interval $\calI$ such that $\bvar_\calI\leq v$, we have for any sub-intervals $\calI_1, \calI_2\subseteq \calI$ and any $\pi\in\Pi$, 
\begin{align*}
\bigabs{\Reg_{\calI_1}(\pi)-\Reg_{\calI_2}(\pi)} \leq 2v.
\end{align*}
\end{lemma}
\begin{proof}
Let $\pi_{\calI_1}^\star=\argmax_{\pi\in\Pi}\calR_{\calI_1}(\pi)$ and $\pi_{\calI_2}^\star=\argmax_{\pi\in\Pi}\calR_{\calI_2}(\pi)$. Then 
\begin{align*}
\Reg_{\calI_1}(\pi)-\Reg_{\calI_2}(\pi)=\calR_{\calI_1}(\pi_{\calI_1}^\star)-\calR_{\calI_1}(\pi)-\calR_{\calI_2}(\pi_{\calI_2}^\star)+\calR_{\calI_2}(\pi).
\end{align*}
By Lemma~\ref{lemma:var-bound}, we have
\begin{align*}
-\var_\calI\leq \calR_{\calI_2}(\pi)-\calR_{\calI_1}(\pi)\leq \var_\calI , 
\end{align*}
and 
\begin{align*}
-\var_\calI\leq \calR_{\calI_1}(\pi_{\calI_2}^\star)-\calR_{\calI_2}(\pi_{\calI_2}^*) \leq \calR_{\calI_1}(\pi_{\calI_1}^\star)-\calR_{\calI_2}(\pi_{\calI_2}^\star)\leq \calR_{\calI_1}(\pi_{\calI_1}^\star)-\calR_{\calI_2}(\pi_{\calI_1}^\star)\leq \var_\calI. 
\end{align*}
Combining them and using $\var_\calI\leq \bvar_\calI$ (Lemma~\ref{lemma:variation_relation}), we get the desired bound.
\end{proof}

\begin{lemma}\label{lemma:variation and TVD}
For any interval $\calI$ such that $\bvar_\calI\leq v$, we have for any sub-intervals $\calI_1, \calI_2\subseteq \calI$, any distribution $Q$ over $\Pi$, and any $\pi\in\Pi$, 
\begin{align*}
\bigabs{V_{\calI_1}(Q,\pi)-V_{\calI_2}(Q,\pi)} \leq \frac{v}{\mu}.
\end{align*}
\end{lemma}
\begin{proof}
For any $s,t\in\calI$ (assuming $s<t$), any $Q$, and $\pi\in\Pi$, 
\begin{align*}
\bigabs{V_s(Q,\pi)-V_t(Q,\pi)} 
& = \Bigg\vert\E_{\calD_s^\calX} \left[ \frac{1}{Q^\mu(\pi(x)|x)}\right] - \E_{\calD_t^\calX}\left[ \frac{1}{Q^\mu(\pi(x)|x)}\right]\Bigg\vert \\
& = \Bigg\vert \int_{\calX} (\calD_s^\calX(x)-\calD_t^\calX(x)) \frac{1}{Q^\mu(\pi(x)|x)}dx \Bigg\vert \\
& \leq \frac{1}{\mu} \int_\calX \bigabs{\calD_s^\calX(x)-\calD_t^\calX(x)}dx \\
& \leq \frac{1}{\mu} \sum_{\tau=s+1}^t \TVD{\calD_\tau - \calD_{\tau-1}} \leq \frac{v}{\mu}. 
\end{align*}
Therefore, 
\begin{align*}
\bigabs{V_{\calI_1}(Q,\pi) - V_{\calI_2}(Q,\pi)} \leq \frac{1}{\abs{\calI_1}}\frac{1}{\abs{\calI_2}} \sum_{s\in\calI_1}\sum_{t\in\calI_2}\bigabs{V_s(Q,\pi)-V_t(Q,\pi)}\leq \frac{v}{\mu}. 
\end{align*}

\end{proof}

\begin{lemma}\label{lem:variance_deviation2}
With probability at least $1 - \delta/4$, \AdaILTCB ensures that for all distributions $Q \in \Delta^\Pi$,
all $\pi \in \Pi$, all intervals $\calI$,
\begin{equation}\label{eq:variance_deviation2}
\whV_{\calI}(Q, \pi) \leq 6.4 V_\calI(Q, \pi) + \frac{80 LK}{|\calI|},
\qquad
V_\calI(Q, \pi) \leq 6.4 \whV_{\calI}(Q, \pi)  + \frac{80 LK}{|\calI|}.
\end{equation}
\end{lemma}
\begin{proof}
This is a consequence of the contexts being drawn independently. A
similar argument of~\citep[Lemma 10]{AgarwalHsKaLaLiSc14} shows that
with probability at least $1 - \delta/4$, the differences
$\whV_{\calI}(Q, \pi) - 6.4 V_\calI(Q, \pi)$ and $V_\calI(Q, \pi) -
6.4 \whV_{\calI}(Q, \pi)$ are both bounded by
\[
\frac{75\ln(N)}{\mu^2 |\calI|} + \frac{6.3\ln(8T^2N^2/\delta)}{\mu|\calI|}
\leq \frac{75 LK}{|\calI|} + \frac{6.3d}{\mu|\calI|}
= \frac{75 LK}{|\calI|} + \frac{6.3 LK\mu}{|\calI|}
\leq \frac{80 LK}{|\calI|},
\]
which completes the proof.
\end{proof}

\begin{lemma}\label{lem:reward_deviation2}
With probability at least $1 - \delta/4$, \AdaILTCB ensures that for all $\pi \in \Pi$
and all intervals $\calI$, 
\begin{equation}\label{eq:reward_deviation2}
|\avgR_\calI(\pi) - \calR_\calI(\pi)| \leq \frac{\mu}{|\calI|\ln(L)} \sum_{t\in\calI} V_t(Q_t, \pi) + \frac{LK\mu}{|\calI|}.
\end{equation}
\end{lemma}

\begin{proof}
By~\citep[Lemma 11]{AgarwalHsKaLaLiSc14}, for any choice of $\lambda \in [0, \mu]$, 
we have with probability at least $1 - \delta/4$, for all $\pi \in \Pi$
and all intervals $\calI$,
\[
|\avgR_\calI(\pi) - \calR_\calI(\pi)| \leq \frac{\lambda}{|\calI|} \sum_{t\in\calI} V_t(Q_t, \pi) + \frac{\ln(8T^2N/\delta)}{\lambda|\calI|}.
\]
Picking $\lambda = \mu/\ln(L)$ and using the fact $\ln(8T^2N/\delta)\ln(L) \leq d$ and $d/\mu = LK\mu$ complete the proof.
\end{proof}

\begin{definition}[$\evt 2$]
Let $\evt 2$ be the event that both Eq.~\eqref{eq:variance_deviation2} and~\eqref{eq:reward_deviation2} hold
for all $\pi \in \Pi$, all intervals $\calI$ and all $Q \in \Delta^\Pi$. This event happens with probability at least $1-\delta/2$.
\end{definition}

Next we prove the following key lemma on the concentration of empirical regrets.

\begin{lemma}\label{lem:regret_deviation2}
Conditioning on $\evt 2$, for any $\pi \in \Pi$, any interval $\calI$ such that $|\calI|\leq L$, $\bvar_\calI\leq v$, and there is no rerun triggered in $\calI$ (i.e., $\forall t\in \calI, \flag_t=\false$), we have
\begin{equation}\label{eq:regret_deviation2}
\Reg_\calI(\pi) \leq 2\whReg_{\calI}(\pi) + \frac{D_1LK\mu}{|\calI|} + D_2v,
\qquad
\whReg_{\calI}(\pi) \leq 2\Reg_\calI(\pi) + \frac{D_1LK\mu}{|\calI|} + D_2v, 
\end{equation}
where $D_1 \triangleq 2\times 10^5$ and $D_2\triangleq 360$. 
\end{lemma}

\begin{proof}
We prove the lemma by induction on the length of $\calI$.  For the
base case $|\calI| = 1$, the bounds hold trivially since both
$\Reg_{\calI}(\pi)$ and $\whReg_{\calI}(\pi)$ are bounded by $1/\mu =
LK\mu/d \leq D_1LK\mu$.  Now assuming that the statement holds for any
$\calI'$ such that $|\calI'| \leq L' < L$, we prove below it holds for
any $\calI$ such that $|\calI| = L' + 1$ too.

Let $\calI=[s,e]$ belong to epoch $i$ and block $j$. For every $t\in [s+1,e]$, define $\ell_t=2^{\floor{\log_2(t-s)}}$ (i.e., $\ell_t$ is the longest $\ell\in\{1,2,4,8,\ldots\}$ such that $[t-\ell, t-1]\subseteq\calI$). Based on the induction assumption, we first prove the property $ V_t(Q_t,\pi)\leq \frac{\Reg_\calI(\pi)}{2\mu} + \order\left(\frac{LK}{t-s+1}+\frac{v}{\mu}\right)$ for all $\pi$: when $t\in[s+1,e]$, 
\begin{align}
& V_t(Q_t,\pi)\leq V_{[t-\ell_t,t-1]}(Q_t,\pi) + \frac{v}{\mu} \tag{by Lemma~\ref{lemma:variation and TVD}} \\
\leq\;& 6.4 \whV_{[t-\ell_t,t-1]}(Q_t, \pi)  + \frac{80 LK}{\ell_t} + \frac{v}{\mu}  \tag{by
    Eq.~\eqref{eq:variance_deviation2}} \\ 
\leq\;& 263 \whV_{[T_i+1, T_i+2^{j-1}-1]}(Q_t, \pi) + \frac{7.76\times 10^3LK}{\ell_t} + \frac{42v}{\mu}  \tag{by Line~\ref{line:ILTCB:check_var}}\\ 
\leq\;& \frac{5.26\times 10^{-4}}{\mu}\whReg_{[T_i+1, T_i+2^{j-1}-1]}(\pi) + \frac{8.29\times 10^3LK}{\ell_t} + \frac{42v}{\mu}\tag{by
  Eq.~\eqref{eq:low_empirical_variance2} and $L \geq \ell_t$}  \\ 
\leq\;& \frac{2.11\times 10^{-3}}{\mu}\whReg_{[t-\ell_t,t-1]}(\pi) +\frac{8.82\times 10^3 LK}{\ell_t} + \frac{43v}{\mu}
\tag{by Line~\ref{line:ILTCB:check_reg}}\\ 
\leq\;& \frac{4.22\times 10^{-3}}{\mu} \Reg_{[t-\ell_t, t-1]}(\pi) + \frac{9.25 \times 10^3LK}{\ell_t} + \frac{44v}{\mu} \tag{by inductive assumption}\\  
\leq\;& \frac{\Reg_\calI(\pi)}{2\mu} + \frac{1.9\times 10^4LK}{t-s+1} + \frac{45v}{\mu}; \tag{
  by Lemma~\ref{lemma:variation and regret} and $t-s+1 \leq 2\ell_t$} \\ \label{eqn:V bounded by reg} 
\end{align}
when $t=s$, Eq.\eqref{eqn:V bounded by reg} also holds because $V_s(Q_s,\pi)\leq \frac{1}{\mu}\leq \frac{\sqrt{LK}}{t-s+1}$.  

Let $\pi_\calI^\star = \argmax_{\pi} \calR_\calI(\pi)$ and
$\whpi_\calI = \argmax_{\pi} \avgR_\calI(\pi)$. We will now establish
the inductive hypothesis. For any $\pi$, $\Reg_\calI(\pi) -
\whReg_{\calI}(\pi)$ is bounded by
\begin{align}
&(\calR_\calI(\pi_\calI^\star) - \calR_\calI(\pi)) - (\avgR_\calI(\pi_\calI^\star) - \avgR_\calI(\pi)) \tag{by optimality of $\whpi_\calI$} \\
\leq\;& \frac{\mu}{|\calI|\ln(L)} \sum_{t\in\calI}\left( V_t(Q_t, \pi) + V_t(Q_t, \pi_\calI^\star) \right) + \frac{2LK\mu}{|\calI|} \tag{by Lemma~\ref{lem:reward_deviation2}} \\
\leq\;& \frac{1}{2}\Reg_\calI(\pi) + \left(\frac{3.8\times 10^4 LK\mu}{|\calI|\ln(L)}\sum_{t\in \calI} \frac{1}{t-s+1}\right) + \frac{2LK\mu}{|\calI|} + 90v
   \tag{by Eq.~\eqref{eqn:V bounded by reg} and $\Reg_\calI(\pi_\calI^\star) = 0$} \\
\leq\;& \frac{1}{2}\Reg_{\calI}(\pi)  + \frac{6\times 10^4LK\mu}{|\calI|} + 90v. \label{eqn:induction hyp 1}
\end{align}
Rearranging proves the first
statement of Eq.~\eqref{eq:regret_deviation2}.  Similarly, we can bound
$\whReg_{\calI}(\pi) - \Reg_\calI(\pi)$ as follows:
\begin{align}
&(\avgR_\calI(\whpi_\calI) - \avgR_\calI(\pi)) - (\calR_\calI(\whpi_\calI) - \calR_\calI(\pi))   \tag{by optimality of $\pi_\calI^\star$} \\
\leq\;& \frac{\mu}{|\calI|\ln(L)} \sum_{t\in\calI}\left( V_t(Q_t, \pi) + V_t(Q_t, \whpi_\calI) \right) + \frac{2LK\mu}{|\calI|} \tag{by Lemma~\ref{lem:reward_deviation2}} \\
\leq\;& \frac{1}{2}(\Reg_\calI(\pi) + \Reg_\calI(\whpi_\calI)) + \left(\frac{3.8\times 10^4 LK\mu}{|\calI|\ln(L)}\sum_{t\in \calI} \frac{1}{t-s+1}\right) + \frac{2LK\mu}{|\calI|} + 90v
   \tag{by Eq.~\eqref{eqn:V bounded by reg} } \\
\leq\;& \frac{1}{2}\Reg_{\calI}(\pi)  + \frac{9.9\times 10^4 LK\mu}{|\calI|} + 180v,
\end{align}
where the last step is by applying Eq.~\eqref{eqn:induction hyp 1} to
$\whpi_\calI$ and using the fact $\whReg_\calI(\whpi_\calI) = 0$.
Rearranging proves the second
statement of Eq.~\eqref{eq:regret_deviation2}, which completes the
induction.
\end{proof}

\begin{lemma}
\label{lemma:at most one rerun ILTCB}
Consider an interval $\calI$ where $\abs{\calI}\leq L$ and $\bvar_\calI \leq v$. If the event $\evt 2$ holds, then there is at most one $t\in \calI$ such that $\flag_t=\true$.
\end{lemma}
\begin{proof}
If there are multiple such time instances, let $t^\p, t\in \calI$ be consecutive ones, and let $i$ and $j$ be the epoch and block indices at time $t$. For $\flag(t)=\true$, there are two possibilities: $t\geq t^\p+L$ or $(j>1 \text{ and }\test(t)=\true)$. 
The former would not happen because $\abs{\calI}\leq L$. 
Thus the latter holds. Since $j>1$, we have $t\geq t^\p+2$. 
By our construction, $\calI^\p\triangleq [t^\p+1, t-1]$ is an interval in which no rerun is triggered, and $\bvar_{\calI^\p}\leq v$ holds. 
Using Lemma~\ref{lem:regret_deviation2} on $\calI^\p$, we have for any $1\leq \ell\leq t-t^\p-1$,
\begin{align*}
\whReg_{[t^\p+1, t^\p+2^{j-1}-1]}(\pi) &\leq 2\Reg_{[t^\p+1, t^\p+2^{j-1}-1]}(\pi) + \frac{2\times 10^5 LK\mu}{2^{j-1}-1} + 360 v
\tag{by Lemma~\ref{lem:regret_deviation2}} \\ 
&\leq 2\Reg_{[t-\ell, t-1]}(\pi) + \frac{2\times 10^5 LK\mu}{2^{j-1}-1} + 364v \tag{by Lemma~\ref{lemma:variation and regret}} \\ 
&\leq 4 \whReg_{[t-\ell, t-1]}(\pi) + \frac{8\times 10^5LK\mu}{\ell} + 1084v, \tag{by Lemma~\ref{lem:regret_deviation2} and $\ell\leq t-t^\p-1\leq 2^{j}-2$}
\end{align*}
\begin{align*}
\whReg_{[t-\ell, t-1]}(\pi) &\leq 2\Reg_{[t-\ell, t-1]}(\pi) + \frac{2\times 10^5 LK\mu}{\ell} + 360 v
\tag{by Lemma~\ref{lem:regret_deviation2}} \\ 
&\leq 2\Reg_{[t^\p+1, t^\p+2^{j-1}-1]}(\pi) + \frac{2\times 10^5 LK\mu}{\ell} + 364v \tag{by Lemma~\ref{lemma:variation and regret}} \\ 
&\leq 4 \whReg_{[t^\p+1, t^\p+2^{j-1}-1]}(\pi) + \frac{1\times 10^6LK\mu}{\ell} + 1084v, \tag{by Lemma~\ref{lem:regret_deviation2} and $\ell\leq t-t^\p-1\leq 2^{j}-2$}
\end{align*}
\begin{align*}
\whV_{[t-\ell, t-1]}(Q, \pi) &\leq 6.4V_{[t-\ell, t-1]}(Q, \pi)  + \frac{80
  LK}{\ell} \tag{by Eq.~\eqref{eq:variance_deviation2}} \\ 
&\leq 6.4V_{[t^\p+1, t^\p+2^{j-1}-1]}(Q, \pi)  + \frac{80
  LK}{\ell} + 6.4v \tag{by Lemma~\ref{lemma:variation and regret}} \\ 
&\leq 41 \whV_{[t^\p+1, t^\p+2^{j-1}-1]}(Q, \pi) + \frac{1104 LK}{\ell} + 6.4v\tag{by Eq.~\eqref{eq:variance_deviation2} and $\ell\leq 2^j-2$}. \\ 
\end{align*}
Therefore, at time $t$, $\test(t)$ should return \true, which contradicts our assumption.   
\end{proof}

We can now prove Theorem~\ref{thm:AdaILTCB2}.

\begin{proof}[Proof of Theorem~\ref{thm:AdaILTCB2}]
Conditioning on $\evt 2$, we can focus on the case when there is no rerun in $\calI$ (i.e., $\forall t\in\calI$, $\flag_t=\false$). This is because by Lemma~\ref{lemma:at most one rerun ILTCB}, there is at most one $t^\p\in\calI$ such that $\flag_{t^\p}=\true$. Suppose this $t^\p$ exists, we can decompose $\calI$ into $\calI_1\cup\{t^\p\}\cup\calI_2$, where $\flag_t=\false$ for all $t\in\calI_1$ or $\calI_2$, and then we can apply our proof to $\calI_1$ and $\calI_2$ separately. The regret in $\calI$ would then be bounded by their sum plus $1$, which is still of the same order. 

With notation $s$, $\ell_t, i, j$ from the proof of Lemma~\ref{lem:regret_deviation2},
for any $t \in \calI$ and $t \geq s+2$,
we have
\begin{align*}
&\sum_{\pi \in \Pi} Q_t(\pi) \Reg_\calI(\pi) \\
&\leq \sum_{\pi \in \Pi} Q_t(\pi) \Reg_{[t-\ell_t, t-1]}(\pi)  + 2v  \tag{by Lemma~\ref{lemma:variation and regret}} \\
&\leq 2 \sum_{\pi \in \Pi} Q_t(\pi) \whReg_{[t-\ell_t, t-1]}(\pi)  + \frac{D_1 LK\mu}{\ell_t} + (D_2+2)v  \tag{by Lemma~\ref{lem:regret_deviation2}} \\
&\leq 2C_1 \sum_{\pi \in \Pi} Q_t(\pi) \whReg_{[T_i+1, T_i+2^{j-1}-1]}(\pi)  + (2C_2 + D_1)\frac{LK\mu}{\ell_t} + (2C_3+D_2+2)v \tag{by Line~\ref{line:ILTCB:check_reg_another}} \\
&\leq 4B C_1  K \mu +  (2C_2 + D_1)\frac{LK\mu}{\ell_t} + (2C_3+D_2+2)v \tag{by Eq.~\eqref{eq:low_empirical_regret2}} \\
&= \scO\left(\frac{LK\mu}{t-s + 1}+v\right). \tag{$L \geq \ell_t \geq (t-s+1)/2$}
\end{align*}
Therefore, the sum of conditional expected regrets $\sum_{t \in \calI} \E_t[ r_t(\pi(x_t)) - r_t(a_t)]$ is bounded by
\begin{align*}
LK\mu + (1-K\mu)\sum_{t\in \calI} \sum_{\pi \in \Pi} Q_t(\pi) \Reg_\calI(\pi)
= \otil(\abs{\calI}v+LK\mu) = \otil\left(\abs{\calI}v+\sqrt{LK\ln(N/\delta)}\right).
\end{align*}
The theorem now follows by an application of the Hoeffding-Azuma inequality.
\end{proof}

\section{Omitted Details for \AdaBIN}\label{app:AdaBIN}
In \AdaBIN, the bin length is set to $H^{\frac{1}{2}}$, and the probability of an exploration bin is $b^{-\frac{1}{2}}$. These two values are not clear before we derive the regret bound and select them optimally. In the following analysis, we will keep them as variables before reaching the final steps. Specifically, we let the bin length be $H^{\gamma}$ (therefore Line~\ref{line:forloop of bin} would be a for-loop from $1$ to $H^{1-\gamma}$, while Line~\ref{line:forloop of bin step} from $1$ to $H^\gamma$); and we let the exploration probability at Line~\ref{line:exploration prob} be $b^{-\theta}$. 

In \AdaBIN, if an interval $\calI$ is an subinterval of an exploration bin, by Freedman's inequality, with probability at least $1-\delta/T^2$, 
\begin{align}
\left| \avgR_{\calI}(\pi) - \calR_{\calI}(\pi) \right| \leq \alpha_{\calI}. \label{eq:concentration4}
\end{align}
For a general interval $\calI$, we have with probability at least $1-\delta/T^2$, 
\begin{align}
\left| \avgR_{\calI}(\pi) - \calR_{\calI}(\pi) \right| \leq \beta_{\calI}. \label{eq:concentration6}
\end{align}
We now define the high probability event that is used in \AdaBIN's analysis: 
\begin{definition}[$\evt 3$]
Define $\evt 3$ to be the following event: for all $\pi\in\Pi$, all interval $\calI\subseteq [1,T]$, Eq.\eqref{eq:concentration4} holds if $\calI$ is an subinterval of an exploration bin; Eq.\eqref {eq:concentration6} holds if otherwise. 
\end{definition}
A union bound over these events implies that $\evt 3$ holds with probability at least $1-\delta/2$. 

In this subsection, we use $S^\p$ to denote the total number of epochs in the whole time horizon, and use $\calE_1, \calE_2, \ldots, \calE_{S^\p}$ to denote individual epochs. Besides, we denote $K^\p=K\ln(N/\delta)$. We also use notations that are defined at the beginning of Appendix~\ref{app:AdaEG2}. 

We analyze \AdaBIN under the switching and drifting distribution settings in the following two subsections respectively. 
\subsection{Switching Regret}
\begin{lemma}
\label{thm: Ada3 not many rerun}
With probability at least $1-\delta/2$, $S^\p\leq S$. 
\end{lemma}
\begin{proof}
It suffices to prove that under $\evt 3$, at any time $t$ if there is no distribution change from the start of the epoch, the algorithm will not rerun. 

Let $t$ be in epoch $i$ and block $j$, and is in an exploration bin. Suppose $\calD_{T_i+1} = \calD_{T_i+2} = \cdots = \calD_{t}$. For any $\ell$ such that $[t-\ell+1, t]$ is a subset of the bin, denoting $A=[t-\ell+1,t]$ and $B=[T_i+1, T_i+2^{j-1}-1]$. When $j>1$, we have for any $\pi$,
\begin{align*}
\avgR_{A}(\pi) & \leq \calR_{A}(\pi)+\alpha_{A} 
\tag{by \eqref{eq:concentration4}} \\
&= \calR_{B}(\pi)+\alpha_{A} 
\tag{the distribution does not change from $T_i$ to $t$}\\
&\leq \avgR_{B}(\pi)+\alpha_{A}+\beta_{B} 
\tag{by \eqref{eq:concentration6}}\\
&\leq \avgR_{B}(\hat{\pi}_{B})+\alpha_{A}+\beta_{B} 
\tag{by the optimality of $\hat{\pi}_B$}\\
&\leq \calR_{B}(\hat{\pi}_{B})+\alpha_{A}+2\beta_{B} 
\tag{by \eqref{eq:concentration6}} \\
&= \calR_{A}(\hat{\pi}_{B})+\alpha_{A}+2\beta_{B} 
\tag{the distribution does not change from $T_i$ to $t$}\\
&\leq \avgR_{A}(\hat{\pi}_{B})+2\alpha_{A}+2\beta_{B}. 
\tag{by \eqref{eq:concentration4}}
\end{align*}
Therefore, $\test(t)$ would return \false. Hence, conditioned on $\evt 3$, the algorithm ends an epoch only when there is some distribution change. This proves the lemma. 

\end{proof}

\begin{definition}[flat bin]
A bin $\calI$ in epoch $i$ and block $j$ is called a \textit{flat bin} if for all $\pi\in\Pi$ and for all $[s,e]$ such that 1) $[s,e] \subseteq \calI$ and 2) $e-s+1=2^q$ for some nonnegative integer $q$, the following holds (with $B(i,j)=[T_i+1,T_i+2^{j-1}-1]$):
\begin{equation}
\calR_{[s,e]}(\pi)\leq \calR_{[s,e]}(\hat{\pi}_{(i,j)}) + 2\beta_{B(i,j)}(\pi)+4\alpha_{[s,e]}.
\end{equation}
\end{definition}
The above definition basically says that in a flat bin, $\hat{\pi}_{(i,j)}$ performs well in the sense that for any $\pi$, $\calR_{[s,e]}(\pi)-\calR_{[s,e]}(\hat{\pi}_{(i,j)})$ is small in all sub-intervals $[s,e]$ such that $e-s+1=2^q$. Since in an exploitation bin, the learner mostly plays $\hat{\pi}_{(i,j)}$, we have the following lemma saying that the regret contributed from \textit{flat exploitation bins} is small. 

\begin{lemma} 
\label{lemma:switch_flat_exploitation}
\AdaBIN always ensures the following:
\begin{align*}
&\sum_{t=1}^T \E_t[r_t(\pi^\star_t(x_t)) - r_t(a_t)] \one\{t \text{ is in flat exploitation bins}\} \\
&\leq \otil\left(\sqrt{K^\p}\left(S^{\p\frac{1}{4}}T^{\frac{3}{4}} + \sqrt{S T}\right) + K^\p\left(\sqrt{S^\p T} + S \right)\right). 
\end{align*}
\end{lemma}

\begin{proof}
The proof will go through several stages: we sequentially calculate the regret in a bin, a block, an epoch, and then the whole time horizon; the regret in a later level is simply a summation over its previous level. Most proofs in this section are all in this form. 
\begin{itemize}[leftmargin=*]
\item \textbf{Regret in a bin.} 
Let $\calI$ be a flat exploitation bin that lies in epoch $i$ and block $j$. Partition $\calI$ into $[s_1, e_1],\ldots, [s_{S_\calI}, e_{S_\calI}]$ such that for every $k\in[S_\calI]$, $[s_k,e_k]$ is an i.i.d. interval. For every $t\in[s_k,e_k]$, define $\ell_t=2^{\log_2\floor{t-s_k+1}}$ (that is, the longest $\ell\in\{1,2,4,8,\ldots\}$ such that $[t-\ell+1, t]\subseteq [s_k,e_k]$). By the definition of flat bin, we have for all $\pi$,
\begin{align*}
\calR_t(\pi)
&=\calR_{[t-\ell_t+1,t]}(\pi) \tag{$[s_k,e_k]$ is i.i.d.}\\
&\leq \calR_{[t-\ell_t+1,t]}(\hat{\pi}_{(i,j)}) + 2\beta_{B(i,j)}+4\alpha_{[t-\ell_t+1,t]} \tag{$\calI$ is a flat bin}\\
&=\calR_t(\hat{\pi}_{(i,j)}) + 2\beta_{B(i,j)}+4\alpha_{[t-\ell_t+1,t]}. \tag{$[s_k,e_k]$ is i.i.d.}
\end{align*}
Therefore, 
\begin{align*}
\calR_t(\pi)-\calR_t(\hat{\pi}_{(i,j)}) &\leq \otil\left(\beta_{B(i,j)}+\sqrt{\frac{K'}{\ell_t}} + \frac{K'}{\ell_t}\right) 
= \otil\left(\beta_{B(i,j)}+\sqrt{\frac{K'}{t-s_k+1}} + \frac{K'}{t-s_k+1}\right). 
\end{align*}
Thus, 
\begin{align*}
&\sum_{t \in \calI} \E_t[ r_t(\pi^\star_t(x_t)) - r_t(a_t)] \\
&\leq \sum_{t \in \calI} (\calR_t(\pi^\star_t(x_t)) - \calR_t(\hat{\pi}_{(i,j)})) + \sum_{t \in \calI} K\mu_t \\
&=\sum_{k=1}^{S_\calI} \sum_{t=s_k}^{e_k}(\calR_t(\pi^\star_t(x_t)) - \calR_t(\hat{\pi}_{(i,j)})) + \sum_{t \in \calI} K\mu_t \\
&\leq \otil\left(\sum_{k=1}^{S_\calI} \sqrt{K'(e_k-s_k+1)} + K' + \sum_{t\in\calI} (\beta_{B(i,j)} + K\mu_t) \right) \\
& \leq \otil\left(\sqrt{K'S_\calI\abs{\calI}} + K'S_\calI + \sum_{t\in\calI} (\beta_{B(i,j)} + K\mu_t) \right), 
\end{align*}
where in the last inequality, we use Cauchy-Schwarz inequality. 

\item \textbf{Regret in a block.}
Now we compute the regret contributed from flat explotation bins in a block $\calJ$ whose epoch and block indices are $i$ and $j$ respectively. Assume there are $\Gamma$ bins in $\calJ$. Then $\abs{\calJ}\leq 2^{j-1}$ and $\Gamma \leq 2^{(j-1)(1-\gamma)}$ by the algorithm. Let $\calI_1, \calI_2, \ldots, \calI_\Gamma$ be the bins in $\calJ$, we have $\sum_{b=1}^{\Gamma}S_{\calI_b}\leq  S_{\calJ} + \Gamma$ (because the boundaries between bins can cut a stationary interval into two). By our conclusion at the previous stage, 
\begin{align*}
&\sum_{t \in \calJ} \E_t[ (r_t(\pi^\star_t(x_t)) - r_t(a_t))]\one\{t \text{ is in flat exploitation bins}\}  \\
&\leq \sum_{b=1}^{\Gamma} \otil\left(\sqrt{K^\p S_{\calI_b}\abs{\calI_b}} + K^\p S_{\calI_b}\right) + \sum_{t\in\calJ} \otil(\beta_{B(i,j)}+K\mu_t) \\
&\leq \otil\left( \sqrt{K^\p (S_{\calJ} + \Gamma) \abs{\calJ} } + K^\p (S_{\calJ}+\Gamma)\right)+ \sum_{t\in\calJ} \otil(\beta_{B(i,j)}+K\mu_t) \tag{Cauchy-Schwarz}\\
& \leq \otil\left(\sqrt{K^\p}\left(\sqrt{S_{\calJ}\abs{\calJ}} + 2^{(j-1)(1-\frac{\gamma}{2})}\right) +K^\p\left(S_{\calJ}+2^{(j-1)(1-\gamma)}\right) \right) + \sum_{t\in\calJ} \otil(\beta_{B(i,j)}+K\mu_t). \tag{$\Gamma\leq 2^{(j-1)(1-\gamma)}$ and $\abs{\calJ}\leq 2^{j-1}$}
\end{align*}
\item \textbf{Regret in an epoch.}
Now we compute the regret in an epoch $\calE$. There are $\ceil{\log_2(1+\abs{\calE})}$ blocks in the epoch $\calE$, and we denote them by $\calJ_1, \calJ_2, \ldots, \calJ_{\ceil{\log_2(1+\abs{\calE})}}$. Similarly, we have $\sum_{j=1}^{\ceil{\log_2(1+\abs{\calE})}} S_{\calJ_j} \leq S_{\calE}+\ceil{\log_2(1+\abs{\calE})}$. Summing up the regret in individual blocks and again using Cauchy-Schwarz inequality, we get
\begin{align*}
&\sum_{t \in \calE} \E_t[ (r_t(\pi^\star_t(x_t)) - r_t(a_t))] \one\{t \text{ is in flat exploitation bins}\} \\
&\leq \otil\left( \sqrt{K^\p} \left( \sqrt{S_\calE\abs{\calE}}   + \sum_{j=1}^{\ceil{\log_2(1+\abs{\calE})}} 2^{(j-1)(1-\frac{\gamma}{2})} \right) 
+ K^\p \left( S_\calE +  \sum_{j=1}^{\ceil{\log_2(1+\abs{\calE})}}  2^{(j-1)(1-\gamma)} \right) \right)  \\
&\ \ \ +\otil\left(\sum_{t\in\calE} (\sqrt{K^\p}(t-T_i)^{-\frac{1}{3}} + K^\p(t-T_i)^{-\frac{2}{3}}) \right)\\
&=\otil\left( \sqrt{K^\p}\left(\sqrt{S_\calE\abs{\calE}} + \abs{\calE}^{1-\frac{\gamma}{2}} + \abs{\calE}^{\frac{2}{3}}  \right) 
+ K^\p\left(S_\calE+ \abs{\calE}^{1-\gamma} + \abs{\calE}^{\frac{1}{3}} \right) \right)\\ 
&=\otil\left( \sqrt{K'}\left(\sqrt{S_\calE \abs{\calE}} + \abs{\calE}^{\frac{3}{4}}\right) + K'\left(S_\calE + \abs{\calE}^{\frac{1}{2}}\right)\right)
\end{align*}

\item \textbf{Regret in the whole time horizon.}
Finally, we sum the bound over epochs and use H\"{o}lder's inequality. Again, we have $\sum_{i=1}^{S^\p} S_{\calE} \leq S + S^\p$. 
\begin{align*}
&\sum_{t=1}^T \E_t[ (r_t(\pi^\star_t(x_t)) - r_t(a_t))]\one\{t \text{ is in flat exploitation bins}\}\\ 
&\leq \otil\left(\sqrt{K^\p} \left( \sqrt{(S+S^\p)T} + S^{\p\frac{1}{4}}T^{\frac{3}{4}} \right) + K^\p (S+S^\p + S^{\p \frac{1}{2}} T^{\frac{1}{2}})\right) \\
&\leq \otil\left(\sqrt{K^\p}\left(S^{\p\frac{1}{4}}T^{\frac{3}{4}} + \sqrt{S T}\right) + K^\p\left(\sqrt{S^\p T} + S \right)\right) .
\end{align*}
\end{itemize}
\end{proof}

Given we have low regret in flat exploitation bins, in the following two lemmas we bound the number of rounds in non-flat bins or exploration bins. 

\begin{lemma}
\label{thm:Ada3 exploration}
With probability at least $1-\delta$, 
\begin{align*}
\sum_{t=1}^T \one\{t \text{ is in exploration bins}\} \leq \otil\left( S^{\p \frac{1}{4}} T^{\frac{3}{4}} \right). 
\end{align*}
\end{lemma}
\begin{proof}
\begin{itemize}[leftmargin=*]
\item \textbf{Regret in a block.}
We first look at a block $\calJ$ whose block index is $j$. Recall that in this block, bin length is set to $2^{(j-1)\gamma}$. Conditioned on all history before $\calJ$, in $\calJ$ we have 
\begin{align*}
\sum_{b=1}^{2^{(j-1)(1-\gamma)}} \E_{\text{bin}(b)} [\one\{\text{bin } b \text{ is exploration}\}] \leq \sum_{b=1}^{2^{(j-1)(1-\gamma)}}b^{-\theta}\leq \order(2^{(j-1)(1-\gamma)(1-\theta)}), 
\end{align*}
where we slightly overload the notation, using $\E_{\text{bin}(b)}$ to denote that expectation conditioned on all history before bin $b$. 
Applying Hoeffiding-Azuma's inequality, with probability at lest $1-\delta/(T\log_2 T)$, the number of exploration bins in $\calJ$ is upper bound by $\otil(2^{(j-1)(1-\gamma)(1-\theta)} + 2^{\frac{1}{2}(j-1)(1-\gamma)})$. In other words, 
\begin{align*}
\sum_{t\in\calJ} \one\{t \text{ is in exploration bins}\}&\leq 2^{(j-1)\gamma}\times \otil(2^{(j-1)(1-\gamma)(1-\theta)} + 2^{\frac{1}{2}(j-1)(1-\gamma)})\\
&=\otil(2^{(j-1)(1-\theta+\gamma\theta)} + 2^{\frac{1}{2}(j-1)(1+\gamma)}).
\end{align*}
Using a union bound, we know that with probability $1-\delta$, the above bound holds for all $j$ and all blocks with index $j$ in the whole time horizon. The $T\log_2 T$ factor is because there can be at most $\log_2 T$ different $j$'s and at most $T$ blocks with index $j$. We call this $\evt 4$. In the following stages, we condition on $\evt 4$. 
\item \textbf{Regret in an epoch.}
Now sum the bound in the previous stage over blocks in an epoch $\calE$. Conditioning on $\evt 4$, we have
\begin{align*}
\sum_{t\in\calE} \one\{t \text{ is in exploration bins}\} &\leq \otil\left(\sum_{j=1}^{\ceil{\log_2(1+\abs{\calE})}} 2^{(j-1)(1-\theta+\gamma\theta)} +2^{\frac{1}{2}(j-1)(1+\gamma)} \right)\\
&=\otil\left(\abs{\calE}^{1-\theta+\gamma\theta} + \abs{\calE}^{\frac{1}{2}+\frac{1}{2}\gamma}\right)=\otil\left(\abs{\calE}^{\frac{3}{4}}\right). 
\end{align*}
\item \textbf{Regret in the whole time horizon.}
Finally, we sum over epochs in the whole time horizon and use union bound. Conditioning on $\evt 4$, we have
\begin{align*}
\sum_{t=1}^T \one\{t \text{ is in exploration bins}\} &\leq \otil\left( \sum_{i=1}^{S^\p} \abs{\calE_i}^{\frac{3}{4}}\right)\leq\otil\left( S^{\p \frac{1}{4}} T^{\frac{3}{4}}\right), 
\end{align*}
where in the final step we use H{\"o}lder's inequality.
\end{itemize}
\end{proof}

\begin{lemma}
\label{thm: Ada3 nonflat}
With probability at least $1-\delta$,
\begin{align*}
\sum_{t=1}^T \one\{ t \text{ is in non-flat bins} \} \leq \otil\left( S^{\p \frac{1}{4}} T^{\frac{3}{4}}\right) . 
\end{align*}
\end{lemma}
\begin{proof}
\begin{itemize}[leftmargin=*]
\item \textbf{Regret in an epoch.}
Let $\calI$ be a non-flat bin whose epoch and block indices are $i$ and $j$. Then there exists some $[s,e]\subseteq \calI$ such that $e-s+1=2^q$ and 
\begin{equation*}
\calR_{[s,e]}(\pi)> \calR_{[s,e]}(\hat{\pi}_{(i,j)}) + 2\beta_{B(i,j)}(\pi)+4\alpha_{[s,e]}.
\end{equation*}
Note that this can only holds for $j>1$. Furthermore, if $\evt 3$ holds and $\calI$ happens to be an exploration bin, then we have 
\begin{gather*}
\bigabs{\calR_{[s,e]}(\pi)-\avgR_{[s,e]}(\pi)}\leq \alpha_{[s,e]},\\
\bigabs{\calR_{[s,e]}(\hat{\pi}_{(i,j)})-\avgR_{[s,e]}(\hat{\pi}_{(i,j)})}\leq \alpha_{[s,e]}.
\end{gather*}
Combining the three inequalities, we get 
\begin{align*}
\avgR_{[s,e]}(\pi)> \avgR_{[s,e]}(\hat{\pi}_{(i,j)}) + 2\beta_{B(i,j)}(\pi)+2\alpha_{[s,e]}.
\end{align*}
This event will make $\test(e)=\true$, which then triggers the rerun. The above argument indicates that as long as $\evt 3$ holds, the non-flat bins in an epoch would only include the first non-flat exploration bins (in which the whole epoch ends) and all non-flat exploitation bins that appear before it. Therefore, the key is to bound the number of non-flat exploitation bins that occur before the first non-flat exploration bin. 

For an epoch $\calE$, let $j^*$ denote the last block index in it. Define $X$ to be the number of non-flat exploitation bins in $\calE$ that appear before the first non-flat exploration bin. 
Note the following two facts: 1) the decision for a bin to be exploration or exploitation is independent of its flatness, 2) a bin with index $b$ is exploitation with probability $1-b^{-\theta}\leq 1-p_{\min}$, where $p_{\min}\triangleq 2^{-(j^*-1)(1-\gamma)\theta}$. Therefore, the probability $\text{Pr}\{X> x\}$ is upper bounded by $(1-p_{\min})^x$. This is because when $X>x$, the first $x$ non-flat bins in the epoch all need to be exploitation bins.
Picking $x$ to be $\frac{\ln (2T/\delta)}{p_{\min}}=2^{(j^*-1)(1-\gamma)\theta}\ln(2T/\delta)\leq \abs{\calE}^{(1-\gamma)\theta}\ln(2T/\delta)$, we get  $\text{Pr}\{X> x\}\leq(1-p_{\min})^x \leq (1/e)^{\ln (2T/\delta)}=\frac{\delta}{2T}$. 

Define $\evt 5$ to be that in every epoch, the quantity $X$ is smaller than $\abs{\calE}^{(1-\gamma)\theta}\ln(2T/\delta)$. Since there are at most $T$ epochs, a union bound guarantees that $\evt 5$ holds with probability at least $1-\delta/2$. 

Thus, when $\evt 3$ and $\evt 5$ both hold, we have
\begin{align*}
\sum_{t\in\calE} \one\{ t \text{ is in non-flat bins} \} \leq \abs{\calE}^{\gamma} \times \otil\left(\abs{\calE}^{(1-\gamma)\theta}\right)=\otil(\abs{\calE}^{\frac{3}{4}}) 
\end{align*}
because the bin length is at most $\abs{\calE}^{\gamma}$. 

\item \textbf{Regret in the whole time horizon.}
Finally we sum this over epochs and use union bound. From the above discussions, with probability at least $1-(1-\text{Pr}(\evt 3))-(1-\text{Pr}(\evt 5))\geq 1-\delta$, 
\begin{align*}
\sum_{t=1}^T \one\{ t \text{ is in non-flat bins} \}\leq \otil\left(\sum_{i=1}^{S^\p}\abs{\calE_i}^{\frac{3}{4}}\right)\leq \otil(S^{\p \frac{1}{4}}T^{\frac{3}{4}}). 
\end{align*}
\end{itemize}
\end{proof}

\begin{proof}[Proof of Theorem~\ref{thm:Ada3 regret} (Part I: switching regret)]
Combining Lemma~\ref{thm: Ada3 not many rerun}, \ref{lemma:switch_flat_exploitation}, \ref{thm:Ada3 exploration}, and~\ref{thm: Ada3 nonflat}, we see that with probability at least $1-5\delta/2$, 
\begin{align*}
\sum_{t=1}^T \E_t[r_t(\pi^\star_t(x_t)) - r_t(a_t)]
&\leq \sum_{t=1}^T \Big( \E_t[r_t(\pi^\star_t(x_t)) - r_t(a_t)] \one\{t \text{ is in flat exploitation bins}\}  \\
&\qquad+ \one\{t \text{ is in exploration bins}\} + \one\{ t \text{ is in non-flat bins} \}  \Big) \\
&\leq \otil\left(\sqrt{K^\p}S^{\frac{1}{4}}T^{\frac{3}{4}}  + K^\p\sqrt{S T}\right). 
\end{align*}
Applying Hoeffding-Azuma inequality shows that with probability at least $1-3\delta$, 
\[
\sum_{t=1}^T r_t(\pi^\star_t(x_t)) - r_t(a_t) = \otil\left(\sqrt{K^\p}S^{\frac{1}{4}}T^{\frac{3}{4}}  + K^\p\sqrt{S T}\right).
\]
\end{proof}

\subsection{Dynamic Regret}
\begin{lemma}
\label{lemma:dynamic S not much}
With probability $1-\delta/2$, $S^\p \leq \otil(1+K^{\p -\frac{2}{5}}\Delta^{\frac{4}{5}}T^{\frac{1}{5}})$.
\end{lemma}
\begin{proof}
Suppose that $\evt 3$ holds. When the $\test(t)$ returns \true at some $t$ in epoch $i$ and  block $j$, we have (let $A=[t-\ell+1,t], B=B(i,j)=[T_i+1, T_i+2^{j-1}-1]$):
\begin{align*}
\avgR_{A}(\hat{\pi}_{A}) > \avgR_{A}(\hat{\pi}_{B})+2\beta_{B} + 4\alpha_{A}. 
\end{align*} 
By the optimality of $\hat{\pi}_{B}$, we have
\begin{align*}
\avgR_{B}(\hat{\pi}_{A}) \leq \avgR_{B}(\hat{\pi}_{B}).
\end{align*}
The above two inequalities indicate for either $\pi=\hat{\pi}_{A}$ or $\pi=\hat{\pi}_{B}$, 
\begin{align*}
\bigabs{\avgR_{A}(\pi) - \avgR_{B}(\pi)} > \beta_{B} + 2\alpha_{A}.
\end{align*}
Since $\evt 3$ holds, 
\begin{gather*}
\bigabs{\avgR_{A}(\pi)-\calR_{A}(\pi)}\leq \alpha_{A},\\
\bigabs{\avgR_{B}(\pi)-\calR_{B}(\pi)}\leq \beta_{B}. 
\end{gather*}
Combining the above three inequalities, we get
\begin{align*}
\bigabs{\calR_{A}(\pi)-\calR_{B}(\pi)}>\alpha_{A}. 
\end{align*}
Since $\bigabs{\calR_{A}(\pi)-\calR_{B}(\pi)}\leq \Delta_{\calE}$ and $\alpha_{A}=\Omega\left(\sqrt{\frac{K^\p}{\ell}}\right)=\Omega(K^{\p \frac{1}{2}}\abs{\calE}^{-\frac{\gamma}{2}})$, we have $\Delta_{\calE}\geq \Omega(K^{\p \frac{1}{2}}\abs{\calE}^{-\frac{\gamma}{2}})$. Now invoke this lower bound for all epochs in which rerun has been triggered (i.e., $\calE_1, \ldots, \calE_{S'-1}$). By H\"{o}lder's inequality, 
\begin{align*}
S^\p-1 
&\leq \left(\sum_{i=1}^{S^\p-1} \abs{\calE_i}^{-\frac{\gamma}{2}} \right)^{\frac{2}{2+\gamma}} \left(\sum_{i=1}^{S^\p-1} \abs{\calE_i} \right)^{\frac{\gamma}{2+\gamma}} \\
&\leq \tilde{\mathcal{O}}\left(K^{\p -\frac{1}{2+\gamma}}\left(\sum_{i=1}^{S'-1}\Delta_{\calE_i} \right)^{\frac{2}{2+\gamma}} T^{\frac{\gamma}{2+\gamma}} \right)  \\
&\leq \tilde{\mathcal{O}}\left(K^{\p -\frac{1}{2+\gamma}} \Delta^{\frac{2}{2+\gamma}} T^{\frac{\gamma}{2+\gamma}} \right)  \\
&= \tilde{\mathcal{O}}\left( K^{\p -\frac{2}{5}}\Delta^{\frac{4}{5}} T^{\frac{1}{5}} \right). 
\end{align*}
\end{proof}

\begin{lemma}
\label{lemma:dynamic flat exploitation}
\AdaBIN always ensures the following
\begin{align*}
&\sum_{t=1}^T\E_t[r_t(\pi_t^\star(x_t))-r_t(a_t)]\one\{t \text{ is in flat-exploitation bins}\}
\leq \otil\left(K^{\p}\Delta^{\frac{1}{3}}T^{\frac{2}{3}} + K^\p S^{\p \frac{1}{4}}T^{\frac{3}{4}} \right).
\end{align*}
\end{lemma}
\begin{proof}
\begin{itemize}[leftmargin=*]
\item \textbf{Regret in a bin.} 
If $\calI$ is an flat exploitation bin in epoch $i$ and block $j$, then for all $[s,e]\subseteq \calI$ such that $e-s+1=2^q$, we have for all $\pi$,
\begin{align*}
\calR_{[s,e]}(\pi)\leq \calR_{[s,e]}(\hat{\pi}_{(i,j)}) + 2\beta_{B(i,j)}+4\alpha_{[s,e]},
\end{align*}
which implies (by expanding the definition of $\calR_{[s,e]}(\pi)$),
\begin{align}
\sum_{t=s}^{e} \E_t[r_t(\pi(x_t))-r_t(a_t)]&\leq \sum_{t=s}^{e} (\calR_t(\pi)-\calR_t(\hat{\pi}_{(i,j)})+K\mu_t) \nonumber \\
&\leq \otil\left( \sqrt{K^\p (e-s+1)} + K^\p + \sum_{t=s}^e (\beta_{B(i,j)}+K\mu_t)\right). \label{eqn:ada3_vbound_single}
\end{align}
Now we divide the whole bin into intervals of length $L'=2^q$ for some integer $q$. Then we can use Lemma~\ref{lem:dynamic2interval} to relate the dynamic regret in the whole bin to the sum of interval regret against a fixed policy on each of the intervals (that is, Eq.~\eqref{eqn:ada3_vbound_single}).

One subtle issue is that there might be one interval (the last one) whose length is less than $L'$. This interval can be further divided into no more than $\log_2L'$ subintervals whose length are all of $2$'s powers. As a whole, there are no more than $\frac{|\calI|}{L'}+\log_2 L'$ intervals each of length no more than $L'$. By Lemma~\ref{lem:dynamic2interval} and Eq.~\eqref{eqn:ada3_vbound_single}, we have
\begin{align*}
\sum_{t\in\calI} \E_t[r_t(\pi_t^\star(x_t))-r_t(a_t)]\leq \otil\left( \left(\frac{|\calI|}{L'} + \log_2 L'\right)(\sqrt{K'L'}+K') + L'\Delta_\calI + \sum_{t\in\calI} (\beta_{B(i,j)}+K\mu_t)\right). 
\end{align*}
Picking $ L'= \min\left\{2^{\floor{\log_2|\calI|}}, 2^{ \big\lfloor \frac{2}{3} \log_2 (|\calI|/\Delta_\calI)\big\rfloor } \right\} $, the right-hand side is further bounded by 
\begin{align*}
\otil\left(  K'|\calI|^{\frac{2}{3}}\var_\calI^{\frac{1}{3}} + K'|\calI|^{\frac{1}{2}} + \sum_{t\in\calI} (\beta_{B(i,j)}+K\mu_t)\right)
\end{align*}

\item \textbf{Regret in an epoch. }
Next, we sum the regret over flat exploitation bins in an epoch. Note there are at most $\otil(\abs{\calE}^{1-\gamma})$ bins in an epoch $\calE$. Using H\"{o}lder's inequality, we have
\begin{align*}
&\sum_{t\in \calE} \E_t[r_t(\pi_t^\star(x_t))-r_t(a_t)]\one\{t \text{ is in flat exploitation bins}\}\\
&\leq \otil\left(K^{\p} \abs{\calE}^{\frac{2}{3}}\Delta_\calE^{\frac{1}{3}} + K^{\p } \abs{\calE}^{1-\frac{\gamma}{2}}  + K^{\p\frac{1}{2}} \abs{\calE}^{\frac{2}{3}} \right) \\
&= \otil\left(K^{\p} \abs{\calE}^{\frac{2}{3}}\Delta_\calE^{\frac{1}{3}} + K^{\p } \abs{\calE}^{\frac{3}{4}} \right).
\end{align*}

\item \textbf{Regret in the whole time horizon. }
Summing over epochs and using H\"{o}lder's inequality, we get 
\begin{align*}
&\sum_{t=1}^T \E_t[r_t(\pi_t^\star(x_t))-r_t(a_t)]\one\{t \text{ is in flat exploitation bins}\}=\otil\left(K^{\p}\Delta^{\frac{1}{3}}T^{\frac{2}{3}} + K^\p S^{\p \frac{1}{4}}T^{\frac{3}{4}} \right). 
\end{align*}

\end{itemize}
\end{proof}

\begin{proof}[Proof of Theorem~\ref{thm:Ada3 regret} (Part II: dynamic regret)]
Combining Lemma~\ref{lemma:dynamic S not much}, \ref{lemma:dynamic flat exploitation}, \ref{thm:Ada3 exploration}, and~\ref{thm: Ada3 nonflat}, we see that with probability at least $1-5\delta/2$, 
\begin{align*}
\sum_{t=1}^T \E_t[r_t(\pi^\star_t(x_t)) - r_t(a_t)]\leq \otil\left(K'\var^{\frac{1}{3}}T^{\frac{2}{3}}  + K^\p\left(1+\var^{\frac{4}{5}}T^{\frac{1}{5}}\right)^{\frac{1}{4}} T^{\frac{3}{4}}\right)\leq \otil\left(K'\var^{\frac{1}{5}}T^{\frac{4}{5}}  + K^\p T^{\frac{3}{4}}\right). 
\end{align*}
Applying Hoeffding-Azuma inequality shows that with probability at least $1-3\delta$, 
\[
\sum_{t=1}^T r_t(\pi^\star_t(x_t)) - r_t(a_t) = \otil\left(K'\var^{\frac{1}{5}}T^{\frac{4}{5}}  + K^\p T^{\frac{3}{4}}\right).
\]
The theorem finally follows by a union bound combining the switching regret bound and the dynamic regret bound we have proven.
\end{proof}

\section{Omitted Proofs in Section~\ref{sec:implications}}\label{app:dynamic2interval}
\begin{proof}[of Lemma~\ref{lem:dynamic2interval}]
It suffices to show that for any $i \in [n]$,
\[
\sum_{t\in\calI_i}\E_t\left[  r_t(\pi^\star_t(x_t)) - r_t(a_t) \right] \leq 
\sum_{t\in\calI_i}\E_t\left[  r_t(\pi^\star_{s_i}(x_t)) - r_t(a_t) \right] + 2|\calI_i|\var_{\calI_i}
\]
The theorem follows by summing up the regrets over all intervals.

Indeed, one can rewrite the regret as follows:
\begin{align*}
\sum_{t\in\calI_i}\E_t\left[  r_t(\pi^\star_t(x_t)) - r_t(a_t) \right]
&=\sum_{t\in\calI_i} \E_t\left[  r_t(\pi^\star_{s_i}(x_t)) - r_t(a_t) \right]
+ \sum_{t\in\calI_i}\E_t\left[ r_t(\pi^\star_t(x_t)) - r_t(\pi^\star_{s_i}(x_t)) \right]\\
&= \sum_{t\in\calI_i} \E_t\left[  r_t(\pi^\star_{s_i}(x_t)) - r_t(a_t) \right]
+ \sum_{t\in\calI_i} \left( \calR_t(\pi^\star_t) - \calR_t(\pi^\star_{s_i}) \right).
\end{align*}
The last term can be further decomposed as:
\[
\sum_{t\in\calI_i} \left(
\calR_{s_i}(\pi^\star_t)  - \calR_{s_i}(\pi^\star_{s_i}) +
\sum_{\tau = s_i+1}^t \left( \calR_\tau(\pi^\star_t) - \calR_{\tau-1}(\pi^\star_t) \right) +
\sum_{\tau = s_i+1}^t \left( \calR_{\tau-1}(\pi^\star_{s_i})  - \calR_\tau(\pi^\star_{s_i}) \right) 
\right)
\]
where $\calR_{s_i}(\pi^\star_t) \leq \calR_{s_i}(\pi^\star_{s_i})$ by definition
and the rest is bounded by $2\var_{\calI_i}$. This finishes the proof.
\end{proof}

\begin{proof}[of Corollary~\ref{cor:dynamic AdaEG}]
The proof of Theorem~\ref{thm:AdaEG2} shows that with probability at
least $1 - \delta/2$, \AdaEG ensures that for any interval $\calI$
such that $|\calI| \leq L$ and $\var_\calI \leq L^{-1/3}$, we have
$\sum_{t \in \calI} \E_t[ r_t(\pi(x_t)) - r_t(a_t)] \leq
\otil\left(|\calI|L^{-\frac{1}{3}} + L^\frac{1}{6}\sqrt{\abs{\calI}K\ln(N/\delta)} +
K\ln(N/\delta)\right)$ for any $\pi$.  We can thus first partition $[1, T]$ evenly
into $T/L$ intervals, then within each interval, further partition it sequentially
into several largest subintervals so that for each of them the variation is at
most $v$. Since the total variation is $\var$, it is clear that
this results in at most $S'\leq T/L + \var/L^{-1/3}$ subintervals (denote them as $\calI_1, \ldots, \calI_{S'}$), each of
which satisfies the conditions of Theorem~\ref{thm:AdaEG2}. Using Lemma~\ref{lem:dynamic2interval}, we get 
\begin{align*}
\sum_{t=1}^T \E_t[ r_t(\pi_t^\star(x_t)) - r_t(a_t)] 
&\leq \sum_{i=1}^{S'} \otil\left(|\calI_i|L^{-\frac{1}{3}} + L^\frac{1}{6}\sqrt{|\calI_i|K\ln(N/\delta)} + K\ln(N/\delta) + |\calI_i|\var_{\calI_i}\right) \\
&\leq \otil\left(\left(\frac{T}{L^{1/3}} + L^{\frac{1}{6}}\sqrt{S'T}+S' \right)K\ln(N/\delta)\right) \\
&\leq \otil\left(\left(\frac{T}{L^{1/3}} + L^{\frac{1}{3}}\sqrt{\Delta T} + \frac{T}{L} + \var L^\frac{1}{3} \right)K\ln(N/\delta) \right) \\
&\leq \otil\left(\left(\frac{T}{L^{1/3}} + L^{\frac{1}{3}}\sqrt{\Delta T} \right)K\ln(N/\delta) \right), 
\end{align*}
where in the second inequality we use Cauchy-Schwarz inequality and in the last one we use the fact $\var \leq T$. Finally using Hoeffding-Azuma inequality leads to the claimed bound. 
\end{proof}

\begin{proof}[of Corollary~\ref{cor:dynamic AdaILTCB}]
The proof follows the same arguments as in Corollary~\ref{cor:dynamic AdaEG} except that now the interval regret is bounded by $\otil\left(\frac{|\calI|}{\sqrt{L}} + \sqrt{LK\ln(N/\delta)} \right)$, and $S^\p\leq T/L + \bvar/L^{-1/2}$. Thus, 
\begin{align*}
\sum_{t=1}^T \E_t[ r_t(\pi_t^\star(x_t)) - r_t(a_t)] 
&\leq \otil\left(\frac{T}{\sqrt{L}} + S' \sqrt{LK\ln(N/\delta)} + \bvar L\right) \\
&\leq \otil\left( \left(\frac{T}{\sqrt{L}} + \bvar L\right)K\ln(N/\delta) \right). 
\end{align*}
Using Hoeffding-Azuma inequality gives the bound.
\end{proof}

\section{Omitted Details for Corralling \bistro}\label{app:BISTRO+}

\LinesNumbered
\SetAlgoVlined
\setcounter{AlgoLine}{0}
\begin{algorithm}[H]
\DontPrintSemicolon
\caption{Corralling \bistro}
\label{alg:corralling_BISTRO+}
\nl {\bf Input}: Contexts $x_1, \ldots, x_T$ and parameter $L$ \\
\nl Define $\gamma = 1/T, \beta = e^{\frac{1}{\ln T}},  \eta = \min\{ \frac{1}{810}, \sqrt{T/\ln N}/(LK)\}, M = \ceil{T\eta} $\\
\nl Initialize $m=1$, $\eta_{1}(i) = \eta$, $\scale_{1}(i) = 2M$ for all $i \in [M]$, $w_1 = \bar{w}_1 = \frac{\one}{M}$,
$q_1 \in \Delta^M$ s.t. $q_1(1) = 1 $ \\
\nl Initialize $\base{1}$, a new copy of \bistro

\nl \For{$t=1$ \KwTo $T$} {
\nl    Receive suggested action $a_t^{i}$ from base algorithm $\base{i}$	 for each $i \in [m]$ \\
\nl    Sample $i_t \sim q_t$, play $a_t = a_t^{i_t}$, receive reward $r_t(a_t)$  \\
\nl    Construct estimated losses $\ell_t(i) = \frac{1-r_t(a_t)}{q_t(i_t)}\one\{i=i_t\} + (1-r_t(a_t))\one\{i > m\},\; \forall i \in [M]$  \label{line:virtual_costs}\\
\nl    Send feedback $\ell_t(i)$ to $\base{i}$ for each $i \in [m]$ \\
\nl    Compute $w_{t+1} \in \Delta^M$ s.t. \[\frac{1}{w_{t+1}(i)} = \frac{1}{w_{t}(i)} + \eta_{t}(i) (\ell_{t}(i) + z_t(i) - \lambda)\] 
      where $\lambda$ is a normalization factor and $z_t(i) = 6\eta_t(i)w_t(i)(\ell_t(i) - (1-r_t(a_t)))^2$ \label{line:LB-OMD} \label{line:Broad-OMD}\\
\nl    Set $\bar{w}_{t+1} = (1 - \gamma) w_{t+1} + \gamma \frac{\one}{M}$   \\
\nl    \For{$i=1$ \KwTo $M$} {
\nl        \lIf{$\frac{1}{\bar{w}_{t+1}(i)} > \scale_{t}(i)$} {set $\scale_{t+1}(i) = \frac{2}{\bar{w}_{t+1}(i)}$, $\eta_{t+1}(i) = \beta\eta_{t}(i)$}
\nl        \lElse{set $\scale_{t+1}(i) = \scale_{t}(i)$, $\eta_{t+1}(i) = \eta_{t}(i)$  \label{line:threshold}}  
    }
\nl    \If{$t$ {\normalfont  is a multiple of} $\lceil T/M\rceil $  \label{line:new_copy}} {
\nl        Update $m \leftarrow m + 1$ \\
\nl        Initialize $\base{m}$, a new copy of \bistro
    }
\nl    Set $q_{t+1}(i) = \frac{\bar{w}_{t+1}(i)}{\sum_{j=1}^m \bar{w}_{t+1}(j)}, \; \forall i \in [m]$
}
\end{algorithm}

We describe the idea of using \corral with \bistro as base algorithms
(see Algorithm~\ref{alg:corralling_BISTRO+} for the pseudocode).
Conceptually we always maintain $M$ base algorithms,
and use \corral almost in a black-box manner as in~\citep{AgarwalLuNeSc17}.
However, crucially the $i$-th copy of the base algorithm only starts
after the end of round $(i-1) \lceil T/M\rceil$, in order to provide regret guarantee
starting from that round (or close to that round).
Therefore, the extra work here is to make sure \corral does not pick algorithms
that have not started, and also to come up with ``virtual rewards'' for algorithms
before they start.

More concretely, at each time we maintain $m \leq M$ copies of the base algorithm 
and a distribution $q_t$ over them
(note that although $q_t$ is in the simplex $\Delta^M \coloneqq \{q \in \fR^M_+: \sum_{i=1}^M q(i)=1 \}$, 
the algorithm always ensure $q_t(i) = 0, \;\forall i > m$). 
First we receive suggested actions $a_t^i$ from each base algorithm $\base{i}$.
Then we sample a base algorithm $i_t \sim q_t$ and play according to its action, that is, $a_t = a_t^{i_t}$.
After receiving its reward $r_t(a_t)$ (or equivalently its cost $1-r_t(a_t)$), 
we construct estimated loss for each of the $M$ algorithms:
for algorithms that have started, this is simply the importance weighted loss;
for algorithms that have not started, this is the actual loss of the picked action (see Line~\ref{line:virtual_costs}).
Next, we send the estimated losses to the $m$ algorithms that have started,
and update several variables that \corral itself maintains, 
including the distributions $w_t$ and $\bar{w}_t$ and the thresholds $\rho_t$ (Line~\ref{line:LB-OMD} to~\ref{line:threshold}).
Finally, we re-normalize the weights $\bar{w}_{t+1}$ over the started algorithms
(including possibly a newly started one) to obtain $q_{t+1}$ and proceed to the next round.

Another additional difference from the original \corral is the way we update $w_t$ (Line~\ref{line:Broad-OMD}).
Here we follow the improved version proposed by~\citet{wei2018more} and incorporate an extra correction term $z_t$ into the loss vector $\ell_t$.
In the original \corral $z_t$ is simply the zero vector.
However, with this more carefully chosen $z_t$ we can eventually improve the bound, replacing some dependence on $T$ by $L$,
as shown in our proof.

~

\begin{proof}[Proof of Theorem \ref{thm:corralling_BISTRO+}]
For any time interval $\calI = [s, t]$ with $|\calI| \leq L$, if $|\calI| \leq T/M$ then the regret bound holds trivially.
Otherwise, there must be a round $s' \in \calI$ such that $s' - s \leq T/M$,
and there is a new copy of \bistro added to the pool at round $s'$.
Denote this new copy by $\base{i^\star}$.
The interval regret on $\calI$ is then clearly bounded by $T/M$ plus the interval regret on $[s', t]$.

Let $c_t(a) = 1 - r_t(a), \;\forall a\in [K]$ and $m_\tau$ be the value of $m$ at round $\tau$ before Line~\ref{line:new_copy}.
Then for any policy $\pi$, we rewrite the interval regret on $[s', t]$ as:
\begin{align*}
\E\left[ \sum_{\tau = s'}^t r_\tau(\pi(x_\tau)) - r_\tau(a_\tau) \right] 
&= \E\left[ \sum_{\tau = s'}^t c_\tau(a_\tau) -\ell_\tau(i^\star) + \ell_\tau(i^\star) - c_\tau(\pi(x_\tau))  \right]  \\
&= \E\left[ \sum_{\tau = 1}^t c_\tau(a_\tau) -\ell_\tau(i^\star) \right] + \E\left[\sum_{\tau = s'}^t c_\tau(a_\tau^{i^\star}) - c_\tau(\pi(x_\tau))  \right]  \\
&= \E\left[ \sum_{\tau = 1}^t \sum_{i=1}^{m_\tau} q_\tau(i) \ell_\tau(i) -\ell_\tau(i^\star) \right] + \E\left[\sum_{\tau = s'}^t c_\tau(a_\tau^{i^\star}) - c_\tau(\pi(x_\tau))  \right]  \\
&= \E\left[ \sum_{\tau = 1}^t \sum_{i=1}^{M} \bar{w}_\tau(i) \ell_\tau(i) -\ell_\tau(i^\star) \right] + \E\left[\sum_{\tau = s'}^t c_\tau(a_\tau^{i^\star}) - c_\tau(\pi(x_\tau))  \right]  \tag{$*$}
\end{align*}
where the second equality uses the fact $c_\tau(a_\tau) = \ell_\tau(i^\star)$ for $\tau < s'$ 
and $\E_{i_\tau \sim q_\tau}[\ell_\tau(i^\star) ] = c_\tau(a_\tau^{i^\star})$ for $\tau \geq s'$,
and the last equality holds because
\begin{align*}
\sum_{i=1}^{M} \bar{w}_\tau(i) \ell_\tau(i) &= 
\left(\sum_{i=1}^{m_\tau} \bar{w}_\tau(i)  \right)  \sum_{i=1}^{m_\tau} q_\tau(i) \ell_\tau(i) + 
\left(\sum_{i=m_\tau+1}^{M} \bar{w}_\tau(i) \right)  \sum_{i=1}^{m_\tau} q_\tau(i) \ell_\tau(i) 
= \sum_{i=1}^{m_\tau} q_\tau(i) \ell_\tau(i).
\end{align*}
Here the first equality follows since $q_\tau(i)\left(\sum_{j=1}^{m_\tau}
\bar{w}_\tau(j) \right) = \bar{w}_\tau(i)$ for $i \leq m_\tau$ and
$\ell_\tau(i) = \sum_{j=1}^{m_\tau} q_\tau(j) \ell_\tau(j) $ for $i >
m_\tau$ by definitions.  

Next we bound the two terms in $(*)$. 
The first term is essentially the regret of the master, corresponding to the update in Line~\ref{line:Broad-OMD}.
Using results from~\citep{wei2018more},\footnote{
This is not explicitly given in~\citep{wei2018more}, but is a direct application of their Theorem~2
with $m_{t,i} = \ell_{t,i_t}$ in their notation. 
}
we obtain
\begin{align*}
&\E\left[ \sum_{\tau = 1}^t \sum_{i=1}^{M} \bar{w}_\tau(i) \ell_\tau(i) -\ell_\tau(i^\star) \right] \\
&\leq \otil\left(\frac{M}{\eta} + \eta\sum_{\tau=1}^t \E\left[w_\tau(i^\star)(\ell_\tau(i^\star) - (1-r_\tau(a_\tau)))^2 \right]  \right) - \E\left[\frac{\rho_{t,i^\star}}{40\eta\ln T}\right] \\
&= \otil\left(\frac{M}{\eta} + \eta\sum_{\tau=s'}^t \E\left[w_\tau(i^\star)(\ell_\tau(i^\star) - (1-r_\tau(a_\tau)))^2 \right]  \right) - \E\left[\frac{\rho_{t,i^\star}}{40\eta\ln T}\right] \\
&\leq \otil\left(\frac{M}{\eta} +  L\eta  \right) - \E\left[\frac{\rho_{t,i^\star}}{40\eta\ln T}\right] 
\end{align*}
where the equality holds because by construction $\ell_\tau(i^\star) = (1-r_\tau(a_\tau))$ for all $\tau < s'$.\footnote{
This is the exact place where we obtain some improvement over the original \corral by using results of~\citep{wei2018more}.
}
For the second term in $(*)$, we apply Lemma~17 of~\citep{AgarwalLuNeSc17} to obtain
\[
\E\left[\sum_{\tau = s'}^t c_\tau(a_\tau^{i^\star}) - c_\tau(\pi(x_\tau))  \right] = \E\left[\rho_{t,i^\star}^{1/3}\right] (LK)^\frac{2}{3}(\ln N)^\frac{1}{3}.
\]
Combining and proceeding similarly as the proof of Theorem~7 of~\citep{AgarwalLuNeSc17}
we have
\begin{align*}
\E\left[ \sum_{\tau = s'}^t r_\tau(\pi(x_\tau)) - r_\tau(a_\tau) \right] 
&\leq \otil\left(\frac{M}{\eta} +  L\eta\right) - \E\left[\frac{\rho_{t,i^\star}}{40\eta\ln T}\right] + 
\E\left[\rho_{t,i^\star}^{1/3}\right] (LK)^\frac{2}{3}(\ln N)^\frac{1}{3}  \\
&\leq \otil\left(\frac{M}{\eta} + L\eta + LK\sqrt{\eta\ln N} \right).
\end{align*}
Adding back the extra $T/M$ term discussed above and plugging in the value of $\eta$ and $M$
lead to $\otil(T^{\frac{1}{4}}(LK)^{\frac{1}{2}}(\ln N)^{\frac{1}{4}} + \sqrt{T})$.
Note that the term $\sqrt{T}$ is dominant only when $L \leq \sqrt{T}$, in which case even the first term is superlinear in $L$ and becomes vacuous.
We can therefore drop the second term and obtain the claimed bound.
\end{proof}

We finally include the dynamic regret guarantee for this algorithm,
which is again a direct application of Lemma~\ref{lem:dynamic2interval} combined with Theorem~\ref{thm:corralling_BISTRO+},
similar to Corollary~\ref{thm:Exp4.S_dynamic}.

\begin{cor}\label{cor:BISTRO+_dynamic}
In the transductive setting, Algorithm~\ref{alg:corralling_BISTRO+} guarantees
\[
\E\left[ \sum_{t=1}^T r_t(\pi_t^\star(x_t)) - r_t(a_t) \right]  = 
\otil\left( \min_{0\leq L' \leq L}\left\{ \frac{T}{L'}\left( T^{\frac{1}{4}}(LK)^{\frac{1}{2}}(\ln N)^{\frac{1}{4}}\right) + \var L' \right\} \right).
\]
If $\var$ is known, optimally setting $L=\min\{ T^{\frac{5}{6}}K^{\frac{1}{3}}(\ln N)^{\frac{1}{6}}/\Delta^{\frac{2}{3}},T\}$ gives
$\otil(\var^\frac{1}{3}T^\frac{5}{6}K^{\frac{1}{3}}(\ln N)^{\frac{1}{6}} + T^{\frac{3}{4}}K^{\frac{1}{2}}(\ln N)^{\frac{1}{4}})$; otherwise, setting $L = T^\frac{5}{6}$
gives $\otil((\sqrt{\var}+1) T^\frac{5}{6}K^\frac{1}{2}(\ln N)^{\frac{1}{4}})$.
\end{cor}

\begin{proof}{\textbf{sketch.}}
The proof follows the same procedure as in Corollary~\ref{thm:Exp4.S_dynamic}: partition $[1,T]$ into $\frac{T}{L'}$ intervals each of length $L'$, plug in the interval regret guarantee for each interval (Theorem~\ref{thm:corralling_BISTRO+}), and then apply Lemma~\ref{lem:dynamic2interval} to obtain the claimed dynamic regret. 
\end{proof}


\section{Interval Regret for \AdaBIN}
\label{appendix:interval_adabin}
\begin{theorem}
Let $\calI$ be an interval with $\Delta_\calI=0$. Then \AdaBIN
with parameter $\delta$ guarantees that with probability at least $1-5\delta$,  
\begin{align*}
\sum_{t\in\calI}r_t(\pi(x_t))-r_t(a_t)=\otil(\sqrt{K'}T^{\frac{3}{4}}+K'\sqrt{T})
\end{align*}
for any $\pi\in\Pi$, where $K'=K\ln(N/\delta)$.
\end{theorem}
\begin{proof}
In the proof of Lemma~\ref{thm: Ada3 not many rerun}, we have shown that with probability at least $1-\delta/2$, if there is no distribution change, an epoch will not rerun. This implies that with probability $1-\delta/2$, the rerun is triggered at most once in $\calI$. Below we assume this event indeed holds. Let $\calI=\calI_1\cup\{t'\}\cup\calI_2$, where in $\calI_1$ and $\calI_2$ rerun is never triggered. 

We can view $\calI_2$ as a fresh epoch with no distribution change in it. Reusing Lemma~\ref{lemma:switch_flat_exploitation}, \ref{thm:Ada3 exploration}, and \ref{thm: Ada3 nonflat}'s \textit{epoch regret} intermediate results (i.e., those \textbf{Regret in an epoch} paragraphs in the proofs) with $S_{\calI_2}=1$, we get 
\begin{align*}
\sum_{t\in \calI_2}\E_t[r_t(\pi(x_t))-r_t(a_t)]=\otil(\sqrt{K'}T^{\frac{3}{4}} + K'\sqrt{T})  
\end{align*}
with probability at least $1-2\delta$.

For $\calI_1$, we can decompose it into $\calJ_{j'} \cup \calJ_{j'+1} \cup \cdots \cup \calJ_{j^*}$, where $\calJ_{j'+1}, \cdots,  \calJ_{j^*-1}$ are complete blocks with block indices $j'+1, \ldots, j^*-1$ respectively, while $\calJ_{j'}$ and $\calJ_{j^*}$ are possibly incomplete blocks (rerun is triggered in $\calJ_{j^*}$). For $j=j'+1, \ldots, j^*$, we can bound the regret in flat exploitation bins in $\calJ_j$ by reusing the \textit{block regret} result in the proof of Lemma~\ref{lemma:switch_flat_exploitation} with $S_{\calJ_j}=1$. Applying the last bound in the \textbf{Regret in a block} part of Lemma~\ref{lemma:switch_flat_exploitation}, the regret in flat exploitation bins in $\calJ_{j'+1} \cup \ldots \cup \calJ_{j^*}$ can be bounded by
\begin{align*}
    &\sum_{t\in \calJ_{j'+1}\cup\cdots\cup\calJ_{j^*}} \E_t[r_t(\pi(x_t))-r_t(a_t)]\one\{t \text{ is in flat-exploitation bins}\}\\
    =&~\otil\left(\sum_{j=j'+1}^{j^*}\sqrt{K^\p}\left(2^{(j-1)\times\frac{3}{4}}\right) +K^\p\left(2^{(j-1)\times\frac{1}{2}}\right) \right) + \otil(\sqrt{K'}T^{\frac{2}{3}} + K'T^{\frac{1}{3}})\\
    =&~\otil(\sqrt{K'}T^{\frac{3}{4}} + K'\sqrt{T}).
\end{align*}
The sum of regret in exploration bins or in non-flat bins can be bounded by $\otil(T^{\frac{3}{4}})$ with probability $1-2\delta$ by Lemma~\ref{thm:Ada3 exploration} and \ref{thm: Ada3 nonflat}'s epoch regret results. Finally, the regret in $\calJ_{j'}$ can be bounded by $|\calJ_{j'}|=\otil(\sqrt{T})$. Combining all above and using Hoeffding-Azuma inequality complete the proof. 
\end{proof}

\end{document}